\documentclass{article}

\usepackage{microtype}
\usepackage{graphicx}
\usepackage{subcaption}
\usepackage{booktabs} 

\usepackage[table]{xcolor} 
\usepackage{colortbl} 

\usepackage{hyperref}

\usepackage[preprint]{icml2026}

\usepackage{fancyhdr}
\fancypagestyle{plain}{%
  \fancyhf{} 
  \fancyfoot[C]{Preprint. Under review.} 
  \renewcommand{\headrulewidth}{0pt} 
  \renewcommand{\footrulewidth}{0pt} 
}

\usepackage{multirow}
\usepackage{siunitx}
\usepackage{amsmath}
\usepackage{amssymb}
\usepackage{mathtools}
\usepackage{amsthm}
\usepackage{pifont}

\definecolor{rangray}{gray}{0.9}

\newcommand{\cmark}{\ding{51}} 
\newcommand{\xmark}{\ding{55}} 

\usepackage[capitalize,noabbrev]{cleveref}

\theoremstyle{plain}
\newtheorem{theorem}{Theorem}[section]
\newtheorem{proposition}[theorem]{Proposition}
\newtheorem{lemma}[theorem]{Lemma}
\newtheorem{corollary}[theorem]{Corollary}
\theoremstyle{definition}

\newtheorem{assumption}[theorem]{Assumption}
\theoremstyle{remark}

\usepackage[most]{tcolorbox}
\newcounter{mythmcount}
\newtcolorbox{thmbox}[1]{
    colback=white, colframe=black, boxrule=0.5pt, sharp corners,
    left=2pt, right=2pt, top=2pt, bottom=2pt,
    before upper={\refstepcounter{mythmcount}\textbf{Theorem \themythmcount~(#1).~~}}
}

\usepackage{tikz}
\usepackage{pgfplots}
\pgfplotsset{compat=1.18}
\usetikzlibrary{arrows.meta, positioning, calc, backgrounds, patterns, intersections}

\definecolor{rangreen}{RGB}{39, 174, 96} 
\definecolor{kanred}{RGB}{231, 76, 60}   
\definecolor{cnnblue}{RGB}{52, 152, 219} 
\definecolor{mlpgray}{RGB}{149, 165, 166} 

\newtcolorbox{propbox}[1]{
    colback=white, colframe=black, boxrule=0.5pt, sharp corners,
    left=2pt, right=2pt, top=2pt, bottom=2pt,
    before upper={\refstepcounter{mythmcount}\textbf{Proposition \themythmcount~(#1).~~}}
}
\usepackage[textsize=tiny]{todonotes}

\icmltitlerunning{Rational ANOVA Networks}

\begin{document}

\twocolumn[
  \icmltitle{Rational ANOVA Networks}


  \begin{icmlauthorlist}
    \icmlauthor{Jusheng Zhang}{sysu}
    \icmlauthor{Ningyuan Liu}{sysu}
    \icmlauthor{Qinhan Lyu}{sysu}
    \icmlauthor{Jing Yang}{sysu}
    \icmlauthor{Keze Wang}{sysu}
  \end{icmlauthorlist}

  \icmlaffiliation{sysu}{Sun Yat-sen University, China}

  \icmlcorrespondingauthor{Jusheng Zhang}{zhangjs@mail.sysu.edu.cn} 

  \icmlkeywords{Machine Learning, ICML, Rational ANOVA Networks}

  \vskip 0.3in
]

\printAffiliationsAndNotice{} 

\begin{abstract}
Deep neural networks typically treat nonlinearities as fixed primitives (e.g., ReLU), limiting both interpretability and the granularity of control over the induced function class. While recent additive models (like KANs) attempt to address this using splines, they often suffer from computational inefficiency and boundary instability. We propose the Rational-ANOVA Network (RAN), a foundational architecture grounded in functional ANOVA decomposition and Padé-style rational approximation. RAN models $f(x)$ as a composition of main effects and sparse pairwise interactions, where each component is parameterized by a stable, learnable rational unit. Crucially, we enforce a strictly positive denominator, which avoids poles and numerical instability while capturing sharp transitions and near-singular behaviors more efficiently than polynomial bases. This ANOVA structure provides an explicit low-order interaction bias for data efficiency and interpretability, while the rational parameterization significantly improves extrapolation. Across controlled function benchmarks and vision classification tasks (e.g., CIFAR-10) under matched parameter and compute budgets, RAN matches or surpasses parameter-matched MLPs and learnable-activation baselines, with better stability and throughput.Code is available at \url{https://github.com/jushengzhang/Rational-ANOVA-Networks.git}.

\end{abstract}

\begin{figure*}[t]
    \centering
    \includegraphics[width=0.88\linewidth]{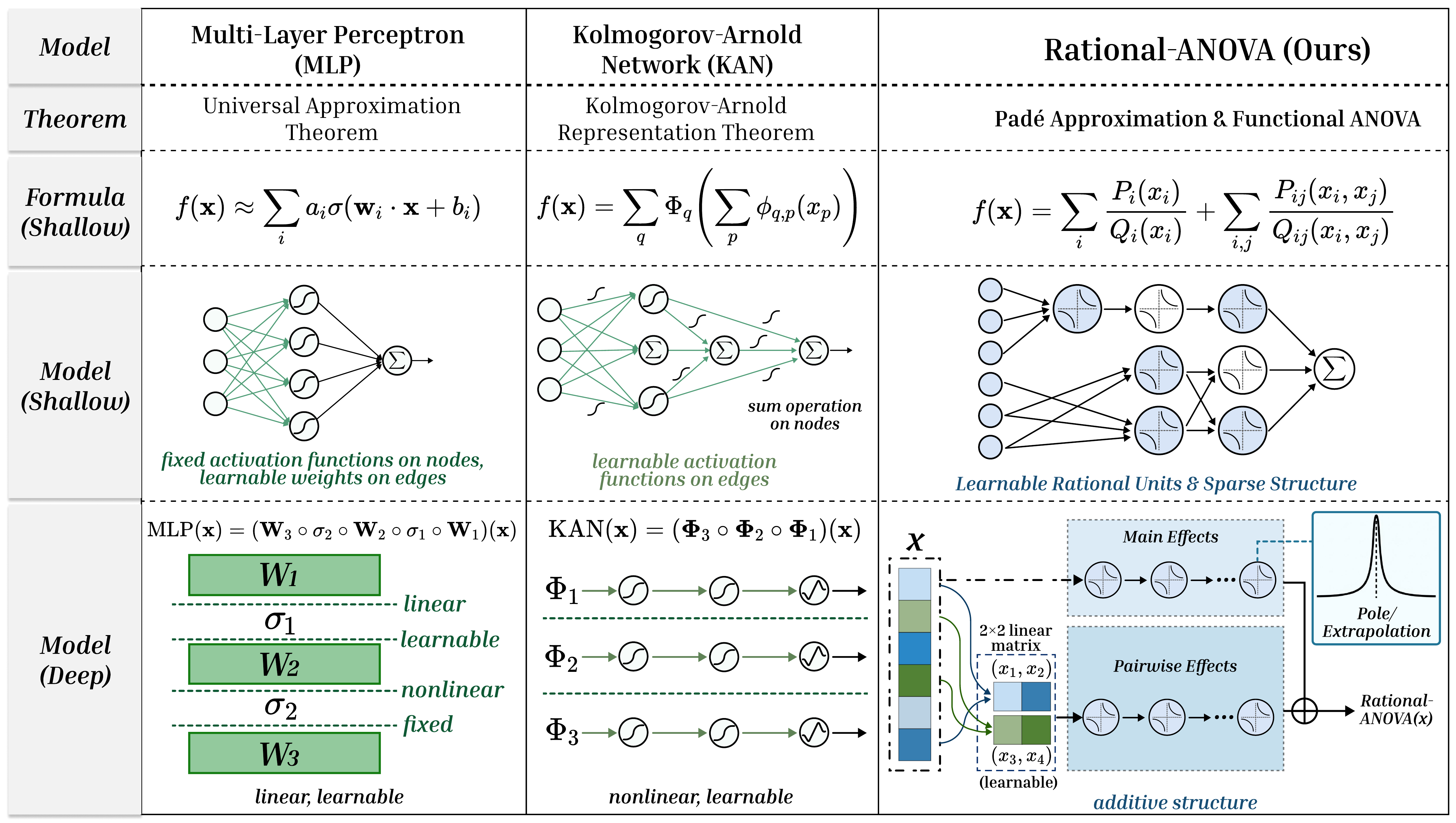}
\caption{\textbf{Comparison of RAN with MLPs and KANs.}
MLPs use fixed activations; KANs learn edge splines.
RAN employs \textbf{learnable rational units} in a \textbf{Functional ANOVA} topology, decomposing $f$ into main effects ($P_i/Q_i$) and sparse interactions ($P_{ij}/Q_{ij}$).}

    \label{fig:comparison}
\end{figure*}

\section{Introduction}
\label{sec:intro}
Deep neural networks have become the default function approximators across domains\cite{Schmidhuber_2015,devore2020neuralnetworkapproximation,lin2021surveytransformers,JMLR:v24:21-1524}, from vision and language to scientific computing. Yet a core ingredient of this success, i.e., nonlinearity, is still treated as a surprisingly fixed and coarse design choice in most architectures. In practice, widely used backbones rely on low-parameter activations (e.g., ReLU~\cite{agarap2018deep, xu2020reluplex}, GELU~\cite{DBLP:journals/corr/HendrycksG16}, SiLU~\cite{DBLP:journals/corr/ElfwingUD17}) whose role is largely to inject generic curvature, while the burden of expressivity is delegated to scaling width and depth. This design is robust and efficient but exposes a persistent limitation: when the target mapping involves sharp transitions, near-singular regimes, or long-tail behaviors, fixed activations often require substantial over-parameterization to fit the shape~\cite{glorot2010understanding, klambauer2017self}. Furthermore, deep composition can amplify \textbf{activation outliers (e.g., heavy-tailed feature excursions)} into unstable optimization dynamics.

A natural response is to make the nonlinearity itself learnable. Recent progress, exemplified by spline-parameterized networks such as Kolmogorov--Arnold Networks (KANs)~\cite{liu2024kan,liu2024kan20kolmogorovarnoldnetworks}, \textbf{suggests} that learnable nonlinearities \textbf{can} improve interpretability and alter the accuracy--efficiency trade-off. However, turning nonlinearity into a high-capacity object raises a practical challenge: deep networks are not merely function classes but carefully tuned parameterizations whose stability depends on activation statistics and residual pathways. A highly expressive nonlinearity may inadvertently degrade conditioning or introduce numerical fragility when composed repeatedly in depth. This tension is especially salient if we seek a base architecture, i.e., a drop-in primitive that can be trained deeply and compared fairly against MLPs and KANs under matched budgets.

In this work, we propose \textbf{Rational-ANOVA Networks (RAN)}, a base architecture that makes nonlinearity learnable while maintaining stable and controllable deep training. RAN is motivated by two classical ideas: Padé-style rational approximation~\cite{DBLP:journals/corr/abs-1907-06732,DBLP:journals/corr/abs-2004-01902, palm2018recurrentrelationalnetworks} and functional ANOVA decomposition. Instead of relying on a fixed pointwise activation, RAN models a multivariate function via a low-order additive structure:
\begin{equation}
    f(\mathbf{x}) \approx \sum_{i=1}^{d} r_i(x_i) + \sum_{(i,j)\in\mathcal{S}} r_{ij}(x_i, x_j),
    \label{eq:ran_formulation}
\end{equation}
where the first term captures main effects and the second term captures sparse pairwise interactions over a controlled set $\mathcal{S}$. \textbf{This set $\mathcal{S}$ can be chosen to match a target compute or parameter budget, yielding a controllable interaction topology.} As shown in Figure~\ref{fig:comparison}, each component $r_i$ and $r_{ij}$ is parameterized by a learnable rational unit---a quotient of low-degree polynomials. Figure~\ref{fig:comparison} contrasts this design with MLPs and KANs: while MLPs place fixed nonlinearities on nodes and KANs place learnable splines on edges, RAN employs learnable rational units within an explicit low-order interaction topology.
A key feature of RAN is its deep compatibility. To ensure numerical stability, we parameterize denominators to be strictly positive (specifically via $1+\text{softplus}(\cdot)$), effectively avoiding poles while retaining the expressive ``Padé-like'' modeling capability~\cite{trefethen2019approximation, telgarsky2017neural}. Furthermore, we employ residual-style gating (shown in Figure~\ref{fig:architecture}), allowing each rational unit to initialize near an identity map ($y \approx x$) and gradually increase its effective nonlinearity during training. This prevents the early-stage instability often caused by aggressive curvature in rational functions.

We evaluate RAN on controlled function benchmarks and physics-inspired problems \textbf{under matched parameter and compute budgets against MLPs and KANs}. Our results show that rational parameterization and sparse ANOVA structure are complementary: rational units handle non-smooth or near-singular regimes, while sparse interactions balance expressivity and generalization.
\begin{figure*}[t]
    \centering
    \includegraphics[width=0.95\linewidth]{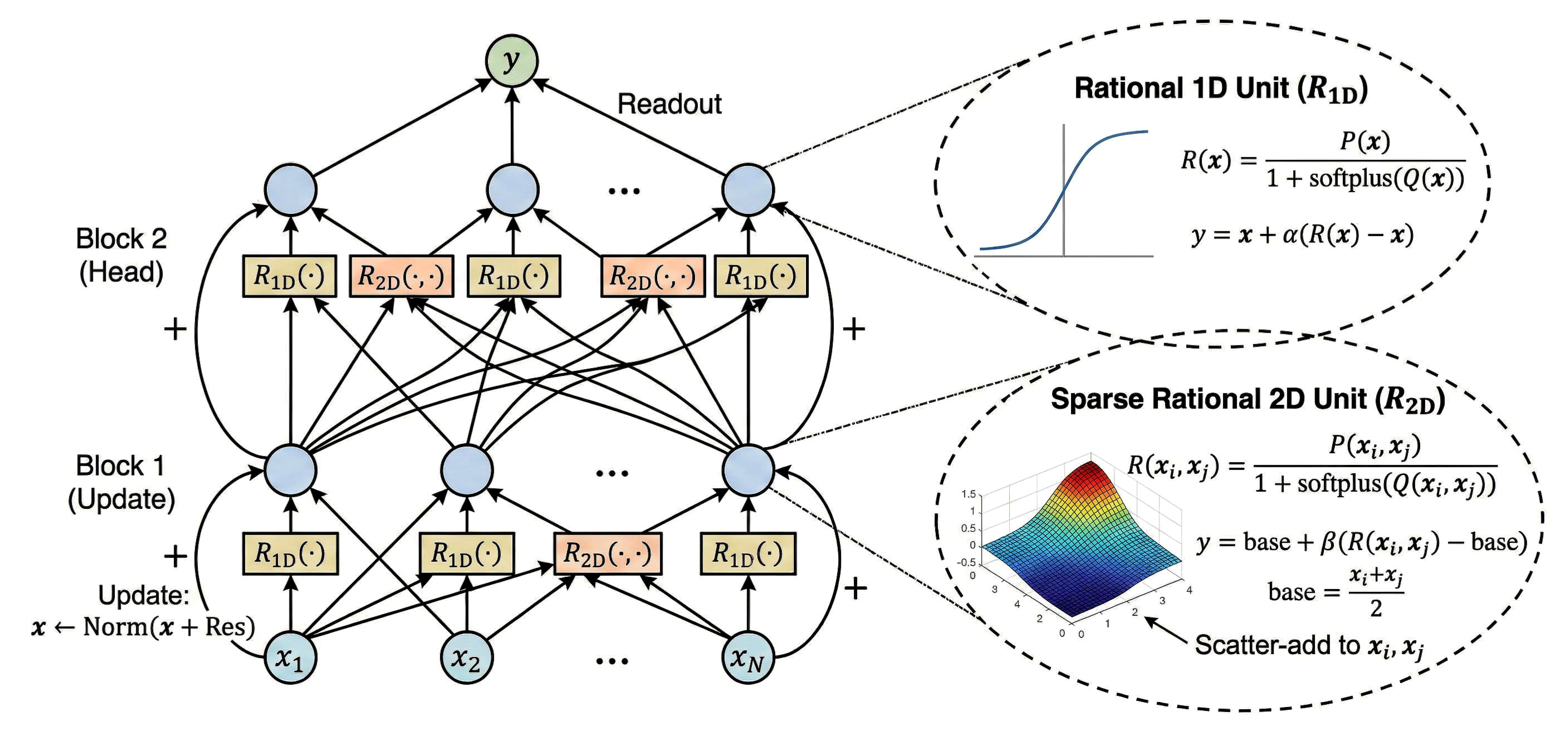}
\caption{\textbf{Deep Rational-ANOVA Network (RAN) architecture.}
\textbf{Left:} A deep backbone of stacked residual blocks; each block performs sparse pairwise message passing (interactions) then node-wise updates.
\textbf{Right:} Learnable rational units. \textbf{$R_{1\text{D}}$} and \textbf{$R_{2\text{D}}$} use residual gating $y=x+\alpha(R(x)-x)$ for identity initialization. Denominators are positive ($1+\text{softplus}(\cdot)$) for stability and pole-free composition.
}
\label{fig:architecture}

\end{figure*}

\section{Rational-ANOVA Networks}
\label{sec:related_method}

\paragraph{Setup and notation.}
Let $\mathbf{x}\in\mathbb{R}^d$ denote the input and $f_\theta(\mathbf{x})$ the model output (logits for classification or a scalar for regression)~\cite{rumelhart1986learning}. We design Rational-ANOVA Networks (RAN) as a \emph{base network primitive} with two design goals: (i) to enable \emph{strictly fair} comparisons with parameter-matched MLPs and KANs, and (ii) to serve as a drop-in replacement for FFNs in pretrained backbones by modifying only the nonlinearity parameterization while preserving linear projections and activation statistics.

\begin{figure*}[t]
    \centering
    \begin{tikzpicture}[
        font=\sffamily\small,
        >=Stealth,
        myblue/.style={blue!60!cyan, thick},  
        mygray/.style={gray!40, thick},       
        myfill/.style={fill=gray!5},          
        myred/.style={red!80!black, ultra thick, ->}, 
        axisstyle/.style={->, black!80, thin},
        labelstyle/.style={font=\footnotesize\bfseries}
    ]

    \coordinate (P1) at (0,0);
    \coordinate (P2) at (5.5,0);
    \coordinate (P3) at (11,0);

    \begin{scope}[shift={(P1)}]
        \node[labelstyle, anchor=south] at (2.2, 3.2) {(a) MLP: Dense Entanglement};
        
        \draw[axisstyle] (0,0) -- (4.5,0); 
        \draw[axisstyle] (0,0) -- (0,3);   
        
        \node[gray, scale=0.8] at (2.25, -0.4) {Input Space (Slice)};
        
        \draw[mygray, name path=old1] plot[smooth, tension=0.7] coordinates {
            (0.2, 0.5) (1.0, 1.2) (2.0, 0.6) (3.0, 1.5) (4.2, 0.5)
        };
        
        \coordinate (Xu1) at (3.0, 1.5);
        \draw[myred] (3.0, 1.5) -- (3.0, 2.5) node[midway, right, scale=0.8] {$\nabla \mathcal{L}$};
        \node[below, red!80!black] at (3.0, 0) {$x_u$};
        \fill[red!80!black] (3.0, 1.5) circle (2pt);
        
        \draw[myblue, name path=new1] plot[smooth, tension=0.7] coordinates {
            (0.2, 0.8) (1.0, 1.6) (2.0, 1.0) (3.0, 2.3) (4.2, 0.7)
        };
        
        \coordinate (Xo1) at (1.0, 0);
        \node[below, black] at (Xo1) {$x_o$};
        \draw[->, dashed, blue] (1.0, 1.2) -- (1.0, 1.6);
        \node[scale=0.7, blue!80!black, align=center] at (1.0, 2.0) {Unintended\\Shift};
        
        \node[align=center, scale=0.75, anchor=north] at (2.25, -0.7) {
            $\Delta f(x_o) \propto K_{\text{dense}}(x_o, x_u)$ \\
            Non-local updates (Global mixing)
        };
    \end{scope}

    \begin{scope}[shift={(P2)}]
        \node[labelstyle, anchor=south] at (2.2, 3.2) {(b) RAN: ANOVA Locality};
        
        \draw[axisstyle] (0,0) -- (4.5,0);
        \draw[axisstyle] (0,0) -- (0,3);
        \node[gray, scale=0.8] at (2.25, -0.4) {Input Space (Slice)};

        \draw[mygray] plot[smooth, tension=0.7] coordinates {
            (0.2, 0.5) (1.0, 1.2) (2.0, 0.6) (3.0, 1.5) (4.2, 0.5)
        };
        
        \draw[myred] (3.0, 1.5) -- (3.0, 2.5) node[midway, right, scale=0.8] {$\nabla \mathcal{L}$};
        \node[below, red!80!black] at (3.0, 0) {$x_u$};
        \fill[red!80!black] (3.0, 1.5) circle (2pt);
        
        \draw[myblue] plot[smooth, tension=0.7] coordinates {
            (0.2, 0.5) (1.0, 1.2) (2.0, 0.7) (3.0, 2.3) (4.2, 0.5)
        };
        
        \coordinate (Xo2) at (1.0, 0);
        \node[below, black] at (Xo2) {$x_o$};
        \node[scale=0.7, green!50!black, align=center] at (1.0, 1.5) {Stable\\(Zero Influence)};
        \draw[green!50!black, thick, shorten >=2pt, shorten <=2pt] (1.0, 1.25) circle (2pt);
        
        \node[align=center, scale=0.75, anchor=north] at (2.25, -0.7) {
            Influence restricted by topology \\
            $\sum K^{\text{main}} + \sum K^{\text{pair}}_{\text{sparse}}$
        };
    \end{scope}

    \begin{scope}[shift={(P3)}]
        \node[labelstyle, anchor=south] at (2.2, 3.2) {(c) Rational Stability};
        
        \draw[axisstyle] (0,0) -- (4.5,0);
        \draw[axisstyle] (0,1.2) -- (0,3); 
        \node[gray, scale=0.8] at (2.25, -0.4) {Pre-activation $x$};
        
        \draw[gray, dashed, thick] (0.2, 0.5) -- (2.0, 0.5) -- (2.2, 2.5) -- (4.2, 2.5);
        \node[gray, scale=0.7, anchor=west] at (3.5, 2.6) {Target};

        \draw[myred] (2.1, 0.8) -- (2.1, 2.2) node[midway, left, scale=0.7] {Squeezing};
        
        \draw[orange!80!black, thick, dotted] plot[smooth, tension=1] coordinates {
            (0.2, 0.6) (1.5, 0.4) (2.1, 1.5) (2.7, 2.8) (3.5, 2.2) (4.2, 2.7)
        };
        \node[orange!80!black, scale=0.7] at (3.8, 1.8) {Poly/Spline};
        
        \draw[myblue, thick] plot[smooth, tension=0.8] coordinates {
            (0.2, 0.55) (1.8, 0.55) (2.4, 2.45) (4.2, 2.45)
        };
        \node[blue!80!black, scale=0.7] at (1.2, 2.0) {\textbf{Rational}};
        
        \draw[->, black, thin] (2.8, 1.0) -- (2.4, 1.5);
        \node[scale=0.65, align=left, fill=white, inner sep=1pt, draw=gray!20] at (3.4, 0.8) {
            Bounded deriv.\\
            $d'(x) \approx 0$ (Eq. 18)
        };
        
        \node[align=center, scale=0.75, anchor=north] at (2.25, -0.7) {
            Sharp transitions w/o poles \\
            Prevents gradient explosion
        };
    \end{scope}

    \end{tikzpicture}
    \caption{\textbf{Learning Dynamics Comparison.} Similar to how RLHF/DPO dynamics affect probability mass, we visualize how structural choices affect function updates $\Delta f$. 
    \textbf{(a) MLP Entanglement:} An update at $x_u$ (red arrow) causes uncontrolled shifts at distant $x_o$ (dashed arrow) due to dense kernel mixing. 
    \textbf{(b) RAN Locality (Ours):} The ANOVA structure (Eq.~\ref{eq:kernel_struct}) disentangles interactions; updates at $x_u$ leave uncorrelated regions $x_o$ stable.
    \textbf{(c) Rational Stability:} Under strong ``squeezing'' gradients (steep transitions), polynomials oscillate (Runge's phenomenon), while RAN's rational units (Eq.~\ref{eq:rational_jacobian_bound}) fit smoothly due to denominator-controlled derivatives.}
    \label{fig:new_dynamics}
\end{figure*}
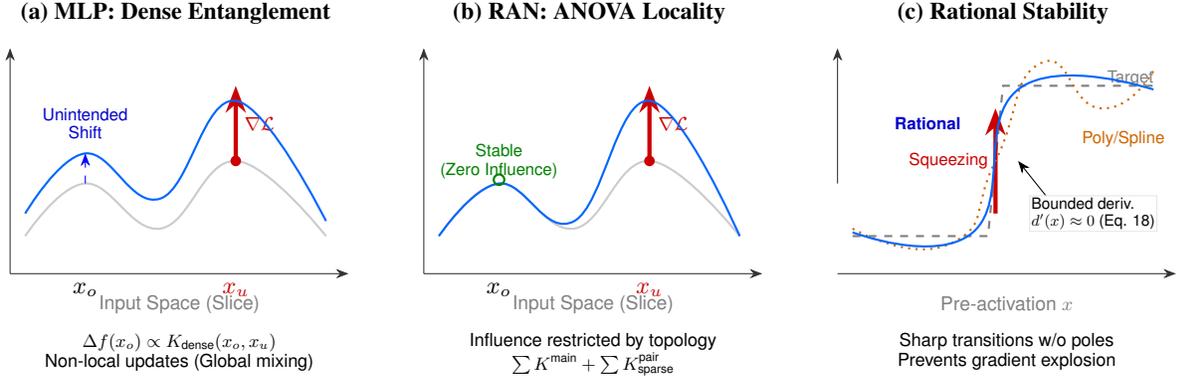
\subsection{ANOVA-Induced Architecture}
\label{sec:anova_arch}

\paragraph{Low-order functional decomposition.}
RAN parameterizes multivariate mappings via an explicit low-order additive structure(conceptually contrasted with MLPs and KANs in Figure~\ref{fig:comparison}):
\begin{equation}
\small
\begin{split}
    f_\theta(\mathbf{x})
    =
    &\underbrace{\sum_{i=1}^d r_i(x_i)}_{\text{main effects}}
    \;+\;
    \underbrace{\sum_{(i,j)\in\mathcal{S}} r_{ij}(x_i,x_j)}_{\text{sparse pairwise effects}} \\
    &\text{(optionally followed by a linear readout).}
\end{split}
\label{eq:ran_core}
\end{equation}
Here, $\mathcal{S}\subseteq \{(i,j):1\le i<j\le d\}$ is a \emph{controlled interaction set}. This design makes the interaction topology \emph{explicit and budgetable}: increasing $|\mathcal{S}|$ adds second-order capacity without the combinatorial explosion of dense pairwise grids. Beyond efficiency, the sparsity of $\mathcal{S}$ also induces a structured influence geometry (as analyzed in Sec.~\ref{sec:dynamics}), allowing cross-sample coupling to be explicitly analyzed and controlled through the interaction topology.

\paragraph{Instantiation for supervised learning.}
For $C$-class classification, we construct a feature vector by concatenating all main and pairwise outputs:
\begin{equation}
\small 
\begin{split}
    \phi_\theta(\mathbf{x})
    = \big[ & r_1(x_1),\dots,r_d(x_d), \\
            & r_{i_1j_1}(x_{i_1},x_{j_1}),\dots,r_{i_Kj_K}(x_{i_K},x_{j_K}) \big] \in\mathbb{R}^{d+K},
\end{split}
\end{equation}
where $K=|\mathcal{S}|$. Logits are produced by a linear head $z(\mathbf{x}) = W\phi_\theta(\mathbf{x}) + b$. This yields a robust baseline where \emph{all nonlinearity lives inside learnable rational units}, while global mixing remains linear (hence stable, interpretable, and compatible with pretrained linear projections if inserted into FFNs).
\textbf{Choosing the interaction set $\mathcal{S}$.}
To ensure fair comparisons and avoid task-specific engineering, our default setting uses a \emph{fixed random} $\mathcal{S}$ sampled once with a global seed and reused across runs. This isolates the inductive bias of rational parameterization combined with sparse low-order structure. We evaluate data-driven choices (e.g., correlation-based or gradient-based selection) only in ablations~(see Appendix~\ref{sec:ablation}).

\subsection{Learnable Rational Units}
\label{sec:rational_units}

\paragraph{1D rational unit (main effects).}
Each main-effect function $r_i:\mathbb{R}\to\mathbb{R}$ is parameterized as a Pad\'e-style rational map:
\begin{equation}
\small 
\begin{split}
    &\tilde r_i(x) = \frac{p_i(x)}{d_i(x)}, \qquad p_i(x)=\sum_{a=0}^{m} \alpha_{ia} x^a, \\
    &d_i(x)=1+\operatorname{softplus}\!\Big(\sum_{b=1}^{n}\beta_{ib}x^b\Big)+\varepsilon,
\end{split}
\label{eq:rational_1d}
\end{equation}
where $(m,n)$ are degrees and $\varepsilon>0$. The constraint $d_i(x)\ge 1+\varepsilon$ prevents poles while retaining the expressive quotient form~(see Appendix~\ref{sec:theoretical_analysis} for regularity proofs). In practice, low degrees are sufficient to capture sharp transitions and saturation effects that fixed activations often model only via increased width or depth.

\paragraph{2D rational unit (pairwise effects).}
Each pairwise function $r_{ij}:\mathbb{R}^2\to\mathbb{R}$ uses a low-degree bivariate rational form:
\begin{equation}
\small 
\begin{split}
    &\tilde r_{ij}(x,y) = \frac{p_{ij}(x,y)}{d_{ij}(x,y)}, \quad p_{ij}(x,y)=\sum_{t=1}^{T}\gamma_{ij,t}\,\psi_t(x,y), \\
    &d_{ij}(x,y)=1+\operatorname{softplus}\!\Big(\sum_{s=1}^{S}\delta_{ij,s}\,\varphi_s(x,y)\Big)+\varepsilon,
\end{split}
\label{eq:rational_2d}
\end{equation}
where $\{\psi_t\}$ and $\{\varphi_s\}$ are simple basis monomials (e.g., $1,x,y,x^2,xy,y^2$). Restricting the basis to low order is crucial for deep composition: it controls curvature and reduces optimization instability, while still providing a richer local shaping mechanism than fixed node-wise activations.

\subsection{Deep Compatibility via Residual Gating}
\label{sec:deep_compat}

\paragraph{Residual-style gating (identity-safe).}
Directly composing expressive rational maps can be unstable if aggressive curvature is activated too early. We therefore wrap each rational map with a residual gate(as detailed in Figure~\ref{fig:architecture}):
\begin{equation}
\small 
    r(x) \;=\; x \;+\; \alpha\cdot\big(\tilde r(x)-x\big),
    \label{eq:gating_1d}
\end{equation}
and similarly for $r_{ij}$ with gate $\alpha_{ij}$. Gates are initialized so that each unit starts near-identity ($r(x)\approx x$), enabling the model to gradually ``turn on'' rational nonlinearity during training.

\paragraph{Stable early-stage Jacobian.}
The Jacobian of the gated unit is
\begin{equation}
\small 
    \frac{\partial r}{\partial x}
    =(1-\alpha) + \alpha \frac{\partial \tilde r}{\partial x}.
    \label{eq:gating_jac}
\end{equation}
A small initial $\alpha$ keeps the effective sensitivity close to $1$. This prevents early training dynamics from being dominated by activation outliers, and ensures deep optimization behavior is comparable to stable residual architectures~(proven in Appendix~\ref{sec:deep_stability}).

\subsection{RAN as a Drop-in FFN Replacement}
\label{sec:ran_ffn}

\paragraph{Replacing node-wise activations.}
A standard Transformer/ViT FFN is $\mathrm{FFN}(h)=W_2\,\sigma(W_1 h)+b$ with $\sigma$ (e.g., GELU). Our minimal drop-in variant replaces $\sigma$ with feature-wise 1D RAN units:
\begin{equation}
\small 
    \mathrm{RAN\text{-}FFN}(h)
    =
    W_2\, r\!\left(W_1 h\right)+b,
    \label{eq:ran_ffn}
\end{equation}
where $r(\cdot)$ is the vectorized collection of gated rational units. This preserves the linear projections (and most pretrained weights) and modifies only the nonlinearity parameterization, while the identity-safe initialization avoids disrupting pretrained activation statistics.

\paragraph{Optional pairwise augmentation.}
If explicit second-order structure is desired inside FFNs, we can augment the hidden representation with sparse pairwise rational features over selected channels:
\begin{equation}
\small 
\begin{split}
    \mathrm{RAN\text{-}FFN^+}(h)
    = W_2\Big(\big[ & r(W_1h), \\
    & \{r_{ij}((W_1h)_i,(W_1h)_j)\}_{(i,j)\in\mathcal{S}_h}\big]\Big)+b.
\end{split}
\end{equation}
While this formulation enhances expressivity by capturing internal feature correlations, it incurs non-trivial computational overhead. Therefore, we position $\mathrm{RAN\text{-}FFN^+}$ primarily as an analytical tool in our ablation studies to investigate the value of explicit second-order information, retaining the 1D variant as the default for efficiency-critical comparisons.
\subsection{Parameter Budgeting and Implementation Details}
\label{sec:complexity}

\paragraph{Budget matching.}
RAN supports exact budget matching with baselines. A 1D unit uses $(m+1)$ numerator coefficients, $n$ denominator coefficients, and one gate parameter, totaling $\approx m+n+2$ parameters per feature. A 2D unit contributes $\approx T+S+1$ parameters per selected pair. Excluding the linear head:
\begin{equation}
    \mathrm{Params} \approx d\big(m+n+2\big) \;+\; |\mathcal{S}|\,(T+S+1).
    \label{eq:param_count}
\end{equation}
This explicit accounting enables strictly matched comparisons with MLPs and KANs.
\textbf{Stability-oriented implementation.}
To ensure robust deep training, we employ the following stabilizers:
\textbf{Positive Denominators:} As defined in Eq.~\eqref{eq:rational_1d}, we enforce $d(\cdot)\ge 1+\varepsilon$ via softplus to strictly avoid poles.
\textbf{Initialization:} We initialize $\tilde r(x)$ near identity by setting $p(x)\approx x$ and $d(x)\approx 1$ (coefficients near zero), and initialize gates $\alpha$ near zero.
\textbf{Optimization:} We optionally use a smaller learning rate for denominator/gate parameters and mild weight decay on denominator coefficients~(detailed in Appendix~\ref{sec:reproducibility}). These stabilizers do not change the hypothesis class but significantly improve numerical stability in deep networks~(see analysis in Appendix~\ref{sec:theoretical_analysis}).

\section{Learning Dynamics of Rational-ANOVA Networks}
\label{sec:dynamics}

\paragraph{Learning dynamics.}
The term ``learning dynamics'' refers to how training updates affect model behavior. We study a concrete \emph{per-step} notion: \emph{how a single gradient update on one training example changes the model's prediction on another input}.
Consider one SGD step at iteration $t$ on $(x_u,y_u)$ with learning rate $\eta$:
\begin{equation}
\small 
\begin{split}
    \Delta\theta_t &\triangleq \theta_{t+1}-\theta_t = -\eta \nabla_\theta \mathcal{L}\!\left(f_{\theta_t}(x_u),y_u\right), \\
    \Delta f_t(x_o) &\triangleq f_{\theta_{t+1}}(x_o)-f_{\theta_t}(x_o).
\end{split}
\label{eq:gd_def}
\end{equation}
The central question is:
\begin{quote}
\emph{After one SGD update on $(x_u,y_u)$, how does the prediction on a different input $x_o$ change?}
\end{quote}

\paragraph{A first-order influence formula.}
A Taylor expansion yields a model-agnostic first-order approximation:
\begin{equation}
\small 
\begin{split}
    \Delta f_t(x_o)
    \approx
    -\eta \,\nabla_\theta f_{\theta_t}(x_o)^\top \nabla_\theta \mathcal{L}\!\left(f_{\theta_t}(x_u),y_u\right).
\end{split}
\label{eq:first_order_generic}
\end{equation}
Thus, the geometry of influence is governed by gradient alignment between the query point $x_o$ and the update point $x_u$.

\paragraph{Classification: eNTK-style decomposition.}
For multi-class classification with logits $z=h_\theta(x)$ and $\pi=\mathrm{Softmax}(z)$, we track the vector of log-probabilities. A first-order expansion gives:
\begin{equation}
\small 
\begin{split}
    \Delta \log \pi_t(\cdot\mid x_o)
    \approx
    -\eta\;
    A_t(x_o)\;
    K_t(x_o,x_u)\;
    G_t(x_u,y_u)
    \;+\;O(\eta^2),
\end{split}
\label{eq:dyn_decomp}
\end{equation}
where $A_t(x_o)=\nabla_z \log \pi_{\theta_t}(\cdot\mid x_o)$ depends only on the current predictive distribution, $G_t(x_u,y_u)$ is the loss gradient w.r.t.\ logits, and
\begin{equation}
\small 
    K_t(x_o,x_u)
    \triangleq 
    \left(\nabla_\theta z_{\theta}(x_o)\vert_{\theta_t}\right)
    \left(\nabla_\theta z_{\theta}(x_u)\vert_{\theta_t}\right)^\top
\end{equation}
is the \textbf{Empirical Neural Tangent Kernel (eNTK)}. Intuitively, $K_t(x_o,x_u)$ measures how strongly an update triggered by $x_u$ projects onto the prediction at $x_o$.

\paragraph{A mild stability assumption.}
For visualization, we assume \emph{relative influence stability}: for a fixed $x_u$, the relative magnitudes of $\|K_t(x_o,x_u)\|$ across different $x_o$ remain approximately stable over short windows. This is weaker than lazy training and is empirically verified in our experiments.

\paragraph{Structured kernel decomposition in RAN.}
As defined in Sec.~\ref{sec:related_method}, RAN parameterizes a multivariate function via a low-order interaction topology (Eq.~\eqref{eq:ran_core}). With logits $z(\mathbf{x})=W\phi_\theta(\mathbf{x})+b$, the eNTK admits an additive decomposition:
\begin{equation}
\small 
    K_t(x_o,x_u)
    =
    K_t^{\mathrm{rat}}(x_o,x_u)
    +
    K_t^{\mathrm{head}}(x_o,x_u).
    \label{eq:K_decomp_head_rat}
\end{equation}
The \emph{head-side} term $K_t^{\mathrm{head}}$ behaves like a linear kernel on features.
\begin{equation}
\small 
\begin{split}
    K_t^{\mathrm{head}}(x_o,x_u) &= (\nabla_{W} z(x_o))(\nabla_{W} z(x_u))^\top \\
    &\quad + (\nabla_{b} z(x_o))(\nabla_{b} z(x_u))^\top,
\end{split}
\label{eq:head_kernel}
\end{equation}
which is low-rank and explicitly controlled by feature similarities. Crucially, owing to the ANOVA topology introduced in Eq.~\eqref{eq:ran_core}, the \emph{rational-unit-side} term decomposes into interpretable groups:
\begin{equation}
\small 
\begin{split}
    K_t^{\mathrm{rat}}(x_o,x_u)
    =
    \sum_{i=1}^{d} K_{t,i}^{\mathrm{main}}(x_o,x_u)
    +
    \sum_{(i,j)\in\mathcal{S}} K_{t,ij}^{\mathrm{pair}}(x_o,x_u).
\end{split}
\label{eq:kernel_struct}
\end{equation}
\textbf{Key insight.} Unlike dense feature mixing in MLPs, RAN makes cross-sample coupling \emph{explicitly controllable}. As visualized in Figure~\ref{fig:new_dynamics} (a) vs.\ (b), dense MLPs suffer from global entanglement where an update at $x_u$ causes unintended shifts at distant $x_o$. In contrast, RAN's ANOVA topology ensures that only inputs overlapping in specific low-order projections (main coordinates and selected pairs in $\mathcal{S}$) can strongly influence each other. This transparency allows us to trade off expressivity and generalization by designing the interaction set $\mathcal{S}$~(analyzed in Appendix~\ref{app:visual_structural_resonance}).

\paragraph{Why rational units are stable in depth.}
For a 1D unit $\tilde r(x)=p(x)/d(x)$, since $d(x) \ge 1+\varepsilon$ (as enforced in Eq.~\eqref{eq:rational_1d}), the quotient rule implies (recalling $d(x)=1+\operatorname{softplus}(q(x))+\varepsilon$):
\begin{equation}
\small
\begin{split}
    \left|\frac{\partial \tilde r}{\partial x}\right|
    &=
    \left|\frac{p'(x)d(x)-p(x)d'(x)}{d(x)^2}\right| \\
    &\le
    |p'(x)| + |p(x)|\,|d'(x)| .
\end{split}
\label{eq:rational_jacobian_bound}
\end{equation}
(Detailed Lipschitz bounds provided in Appendix~\ref{sec:theoretical_analysis}.) Moreover, denominator sensitivity is controlled by
\begin{equation}
\small
    d'(x) = \sigma_{\mathrm{sig}}\!\big(q(x)\big)\,q'(x), \quad \text{with } \sigma_{\mathrm{sig}}(\cdot)\in(0,1),
    \label{eq:den_deriv_control}
\end{equation}
which prevents the denominator pathway from amplifying gradients excessively. 
Figure~\ref{fig:new_dynamics} (c) visualizes this benefit: while polynomials often oscillate under steep gradient requirements (Runge's phenomenon), RAN's rational units maintain smooth and bounded updates due to this denominator-controlled mechanism, thereby enabling stable deep training~(proven in Appendix~\ref{sec:deep_stability}).

\paragraph{Accumulated influence.}
To study long-term effects, we visualize the accumulated influence from a subset of updates $\mathcal{U}$ onto a query input $x_o$:
\begin{equation}
\small 
\begin{split}
    \mathrm{Infl}(x_o;\mathcal{U})
    \approx
    -\eta \sum_{t\in\mathcal{U}} A_t(x_o)\,K_t(x_o,x_{u(t)})\,G_t(x_{u(t)},y_{u(t)}).
\end{split}
\label{eq:accum_infl}
\end{equation}
In Sec.~\ref{sec:experiments}, we show how RAN's sparse interaction topology reshapes these influence patterns compared to MLPs and KANs.


\begin{table*}[htbp]
\centering
\caption{Comparison of Model Accuracy across Different Datasets. \textbf{RAN} (ours) is highlighted with a gray background. Note that RAN achieves consistent improvements across varying parameter scales (50k to 1.0M) and diverse datasets.}
\label{tab:model_comparison_v7_updated}
\resizebox{0.8 \textwidth}{!}{%
\begin{tabular}{ll cccc c ll cccc}
\toprule
\multirow{2}{*}{\textbf{DATASET}} & \multirow{2}{*}{\textbf{PARAMS}} & \multicolumn{4}{c}{\textbf{Accuracy (\%)}} & & 
\multirow{2}{*}{\textbf{DATASET}} & \multirow{2}{*}{\textbf{PARAMS}} & \multicolumn{4}{c}{\textbf{Accuracy (\%)}} \\
\cmidrule(lr){3-6} \cmidrule(lr){10-13}
 & & KAF & MLP & KAN & \multicolumn{1}{c}{\cellcolor{rangray}\textbf{RAN$^\dagger$}} & & 
 & & KAF & MLP & KAN & \multicolumn{1}{c}{\cellcolor{rangray}\textbf{RAN$^\dagger$}} \\
\midrule
\multirow{5}{*}{\textbf{MNIST}} 
 & 50k   & 97.45 & \textbf{97.60} & 96.50 & \cellcolor{rangray}97.55 & & \multirow{5}{*}{\textbf{EMNIST-Let}} 
 & 50k   & 82.50 & \textbf{84.60} & 70.00 & \cellcolor{rangray}84.55 \\
 & 100k  & 98.15 & \textbf{98.25} & 97.10 & \cellcolor{rangray}98.20 & & 
 & 100k  & 86.80 & \textbf{88.55} & 78.50 & \cellcolor{rangray}{88.40} \\
 & 200k  & 98.35 & \textbf{98.45} & 97.40 & \cellcolor{rangray}98.40 & & 
 & 200k  & 88.90 & \textbf{90.20} & 81.20 & \cellcolor{rangray}90.15 \\
 & 300k  & 98.50 & \textbf{98.60} & 97.50 & \cellcolor{rangray}98.55 & & 
 & 300k  & 89.15 & 90.80 & 81.50 & \cellcolor{rangray}\textbf{90.85} \\
 & 400k  & 98.65 & 98.70 & 97.65 & \cellcolor{rangray}\textbf{98.75} & & 
 & 400k  & 89.40 & \textbf{91.30} & 82.00 & \cellcolor{rangray}91.25 \\
\cmidrule(lr){1-6} \cmidrule(lr){8-13}
\multirow{5}{*}{\textbf{FMNIST}} 
 & 50k   & 88.00 & 88.50 & 86.00 & \cellcolor{rangray}\textbf{95.39} & & \multirow{5}{*}{\textbf{SVHN}} 
 & 200k  & 76.50 & \textbf{79.20} & 60.50 & \cellcolor{rangray}79.15 \\
 & 100k  & 89.20 & 89.00 & 88.00 & \cellcolor{rangray}\textbf{96.67} & & 
 & 500k  & 79.80 & 81.55 & 63.80 & \cellcolor{rangray}\textbf{81.65} \\
 & 200k  & 89.50 & 89.20 & 87.00 & \cellcolor{rangray}\textbf{97.49} & & 
 & 1M    & 81.20 & \textbf{82.40} & 62.00 & \cellcolor{rangray}82.35 \\
 & 300k  & 89.50 & 89.30 & 86.50 & \cellcolor{rangray}\textbf{97.69} & & 
 & 1.5M  & 81.80 & 83.15 & 55.00 & \cellcolor{rangray}\textbf{83.25} \\
 & 400k  & 89.50 & 89.30 & 86.20 & \cellcolor{rangray}\textbf{97.79} & & 
 & 2M    & 82.05 & \textbf{83.85} & 48.00 & \cellcolor{rangray}83.80 \\
\cmidrule(lr){1-6} \cmidrule(lr){8-13}
\multirow{2}{*}{\textbf{CIFAR-10}} 
 & 0.5M  & 56.20 & 54.98 & 46.81 & \cellcolor{rangray}\textbf{58.84} & & \multirow{2}{*}{\textbf{CIFAR-100}} 
 & 0.5M  & 25.62 & 25.85 & 17.73 & \cellcolor{rangray}\textbf{27.86} \\
 & 1.0M  & 56.95 & 56.45 & 43.32 & \cellcolor{rangray}\textbf{59.05} & & 
 & 1.0M  & 26.75 & 27.10 & 14.80 & \cellcolor{rangray}\textbf{28.12} \\
\bottomrule
\end{tabular}
}
\end{table*}

\section{Experiments}
\label{sec:experiments}

Unless otherwise stated, all experiments are conducted under the strict premise of \textbf{matched parameter counts} (and matched FLOPs when applicable), and all runs are performed on NVIDIA RTX 6000 GPUs.

\textbf{Evaluation roadmap and why we start from vision.}
Our experiments are organized to support RAN as a \emph{drop-in nonlinear block} across regimes.
We begin with \textbf{controlled visual benchmarks} because they provide a widely adopted and reproducible stress test for both expressivity and optimization stability under strict parameter/FLOPs budgets.
Importantly, this suite is also the standard testbed in the \textit{Kanbefair} protocol, enabling direct comparisons to learnable-nonlinearity baselines under identical pipelines.
To avoid conclusions that depend on a single domain, we further include \textbf{(i) large-scale backbone integration} (ViT on ImageNet-1K), 

\textbf{(ii) real-world restoration} (PolyU denoising), and \textbf{(iii) non-visual evaluations and ablations} (e.g., TabArena) to isolate the roles of topology and rational nonlinearities.
\textbf{Broader evaluation across domains.}
We emphasize that this paper does \emph{not} claim KAN is primarily designed for vision. Instead, we adopt vision benchmarks as a standardized, budget-controlled protocol for \emph{learnable nonlinearity} comparisons.
To fairly reflect the regimes where KAN is typically most competitive, we additionally evaluate RAN on \textbf{scientific discovery and extrapolation settings} where KAN is often expected to excel:
(i) \textbf{tabular learning} with strong efficiency constraints (TabArena; see Supplementary Section~\ref{subsec:tabarena_eval}),
(ii) \textbf{symbolic regression / physical law recovery} (e.g., Lorentzian potential and Feynman equations; Supplementary Sections~\ref{sec:physics_experiment} and~\ref{sec:feynman_extensive}),
and (iii) \textbf{extrapolation and boundary stability} via the Runge stress test (Supplementary Section~\ref{sec:extrapolation_analysis}).
These tasks directly evaluate interpretability, structure recovery, and out-of-support generalization, reducing reliance on vision-only evidence.
\subsection{Comprehensive Evaluation on Visual Benchmarks}
\label{subsec:vision_bench}

\paragraph{Benchmarks and baselines.}
We follow the \textit{Kanbefair} evaluation framework~\cite{yu2024kan} to benchmark across diverse visual classification datasets,
including MNIST~\cite{lecun2002gradient}, EMNIST~\cite{cohen2017emnist}, FMNIST~\cite{xiao2017fashion}, SVHN~\cite{netzer2011reading}, and CIFAR~\cite{krizhevsky2009learning}.
We compare against: (i) \textbf{MLP} with GELU (and ReLU as an additional reference), (ii) the original \textbf{KAN}~\cite{liu2024kan}, and (iii) the recent \textbf{KAF}~\cite{zhang2025kolmogorov}, in Table~\ref{tab:model_comparison_v7_updated}.


\subsubsection{Training protocol and fairness controls}
\label{subsubsec:fairness_protocol}

\paragraph{Unified training budget and model selection.}
To avoid confounding factors from different optimization schedules, we train \textbf{all methods under a unified training budget}.
Model selection is performed using the \textbf{validation set} (or a held-out split following Kanbefair), and we report the corresponding \textbf{test accuracy} of the selected checkpoint.
This prevents inadvertently favoring any method via ``test-set best'' selection and makes the comparison reproducible.

\paragraph{Parameter matching.}
For each dataset and each target budget, we enforce \textbf{parameter parity} across methods (within a tight tolerance).
For KAN, we sweep the recommended configuration space (grid sizes $\{3,5,10,20\}$ and B-spline degrees $\{2,3,5\}$) and report the best-performing setting under the same parameter budget.
For MLP and KAF, we follow their standard implementations and tune width/depth only to meet the required parameter count.

For RAN, parameter control is non-trivial because the sparse interaction set contributes additional learnable components.
We therefore use an explicit parameter estimator to guide architecture instantiation:
$
\mathrm{Params} \approx (18 + C)N + (26 + C)K + C + 2,
$
where $N$ denotes the hidden width, $K$ is the rational grid/order size, and $C$ is the number of classes.
Crucially, we \textbf{dynamically adjust the number of sparse interaction pairs} $|\mathcal{S}|$ so that RAN matches the target budget of each baseline configuration.
This ensures that improvements cannot be attributed to a larger parameter count.

\paragraph{Model capacity sweep.}
To stress-test robustness across regimes, we sweep a wide range of hidden widths (from 2 to 1024 when feasible) and report performance at multiple parameter budgets.
Unless otherwise stated, we report results averaged over multiple random seeds (mean and variance in the supplementary), and keep all data preprocessing and augmentation identical across methods.

\subsection{Evaluation in Large-Scale Vision Models}
\label{subsec:vit_eval}

To assess scalability, we integrate RAN into Vision Transformers~\cite{DBLP:journals/corr/abs-2010-11929}.
Concretely, we adopt ViT-Tiny (DeiT-T)~\cite{pmlr-v139-touvron21a} and \textbf{replace the standard FFN blocks} with \textbf{RAN-FFN} blocks while preserving the original \textbf{parameter budget} and \textbf{FLOPs}.
All models are trained on ImageNet-1K~\cite{deng2009imagenet, russakovsky2015imagenet} using the same training recipe (including standard strong augmentations such as Mixup and CutMix, and the same learning-rate schedule) to ensure fair comparison.

\paragraph{Baselines: }
(i) the standard MLP-based ViT (GELU FFN), (ii) KAF, and (iii) KAN.
For KAN, we attempted to instantiate the smallest feasible configurations under the ViT embedding dimension while enforcing the same budget constraints.
However, due to the intrinsic parameter/memory growth induced by the B-spline basis expansion in high-dimensional settings, KAN remains \textbf{infeasible} on our hardware budget (48GB GPU memory), resulting in \texttt{OOM} under budget-matched conditions.
We explicitly mark this outcome as \texttt{OOM} rather than reporting an under-budget or structurally altered KAN that would not be comparable.

\begin{table}[t]
\centering
\setlength{\tabcolsep}{4.5pt}
\caption{\textbf{Performance on ViT-T/16 (ImageNet-1K).} All methods are trained with the same recipe; Params/FLOPs are matched to the baseline when applicable. \texttt{OOM} denotes infeasibility under the budget-matched setting on 48GB GPUs.}
\label{tab:vit_comparison}
\resizebox{\columnwidth}{!}{%
\begin{sc}
\begin{small}
\begin{tabular}{l l c c c c}
\toprule
\textbf{Method} & \textbf{Mechanism} & \textbf{Params} & \textbf{FLOPs} & \textbf{Top-1 (\%)} & \textbf{Gain} \\
\midrule
\multicolumn{6}{l}{\textit{\textbf{Existing Approaches}}} \\
(1) MLP (Baseline) & Linear + GELU      & 5.7 M  & 1.08 G & 72.3 & -- \\
(2) KAN            & B-Spline           & \textit{OOM} & \textit{OOM} & -- & -- \\
(3) KAF            & Kernel Function    & 5.9 M  & 1.12 G & 73.2 & +0.9 \\
\midrule
\multicolumn{6}{l}{\textit{\textbf{Proposed Method}}} \\
\rowcolor{rangray}
(4) \textbf{RAN (Ours)} & \textbf{Rational} & \textbf{5.7 M} & \textbf{1.08 G} & \textbf{74.2} & \textbf{+1.9} \\
\bottomrule
\end{tabular}
\end{small}
\end{sc}
}
\end{table}

Overall, RAN improves Top-1 accuracy from 72.3\% (MLP baseline) to \textbf{74.2\%} under the same Params/FLOPs.
We attribute the gain to the combination of (i) \textbf{learnable rational nonlinearities} with stronger local approximation capacity and (ii) \textbf{near-identity initialization and stabilization constraints} that preserve optimization stability when replacing deep FFN stacks.

\subsection{Real-World Image Denoising on PolyU}
\label{sec:exp_polyu}

\begin{figure}[t]
    \centering
    \includegraphics[width=0.85\linewidth]{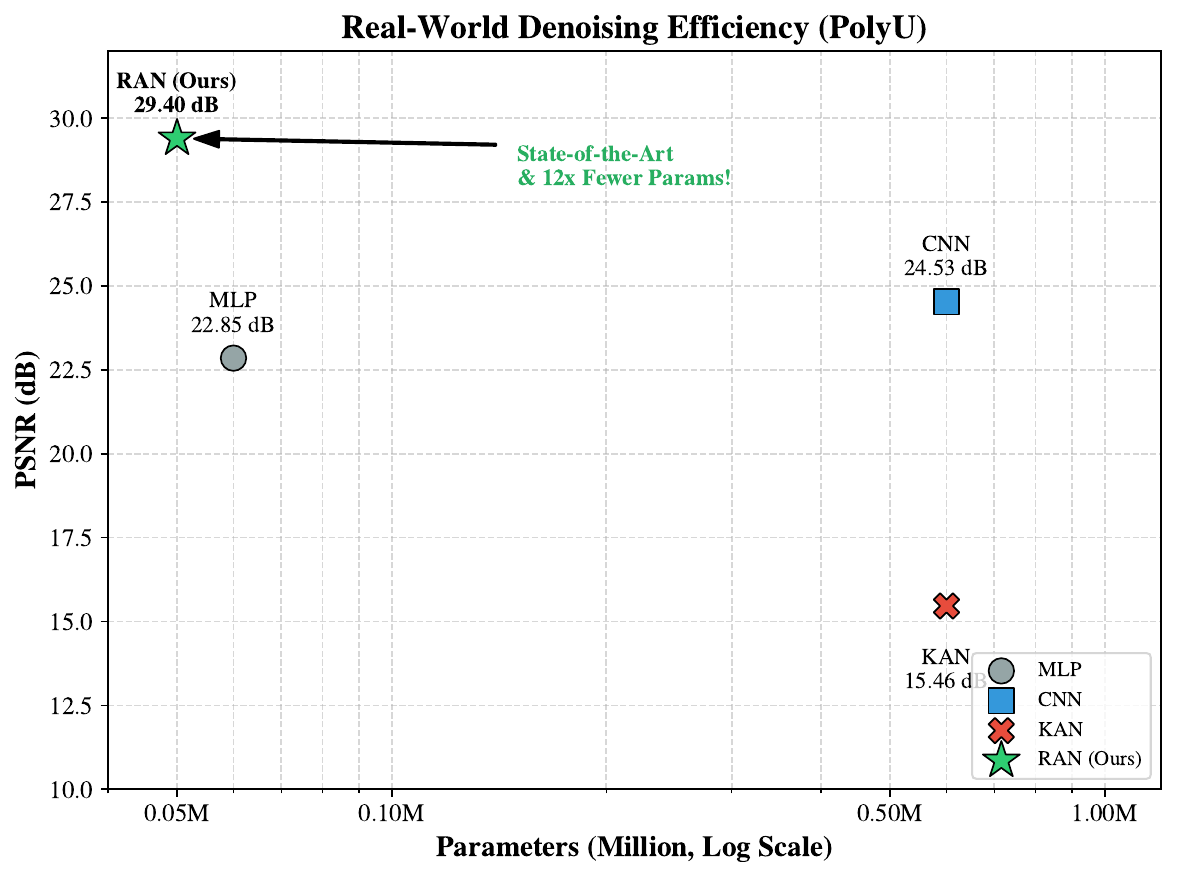}
    \vspace{-0.1in}
    \caption{\textbf{Real-World Denoising Efficiency on PolyU.} PSNR (y-axis) vs. parameter count (x-axis, log scale). Each point corresponds to a budgeted model instance. RAN achieves a strong accuracy--efficiency trade-off and lies on the Pareto frontier.}
    \label{fig:polyu_efficiency}
    \vspace{-0.1in}
\end{figure}
We further evaluate RAN on the \textbf{PolyU Real-world Noisy Image Dataset} (Xu et al., 2018), which contains real noise under varying ISO and shutter speeds.
Unlike synthetic Gaussian noise, PolyU requires modeling complex, non-i.i.d.\ noise patterns under tight budgets.

\paragraph{Protocol and preprocessing.}
We split the dataset into train/test following paired \texttt{Real} (noisy) and \texttt{Mean} (clean) folders.
During training, we apply random cropping to generate $128\times128$ patches; during evaluation, we use center crops for deterministic reporting.
To apply our non-convolutional model, we use a patch-based strategy: each image is unfolded into non-overlapping patches of size $P\times P$ ($P=4$), yielding $d = 3P^2 = 48$ input dimensions per patch.
\textbf{Model configuration and budgets.}
We instantiate a lightweight \textbf{Residual Rational-ANOVA} model with 3 layers and a hidden dimension of 16, equipped with both unit-level gating and block-level residual connections for stable optimization.
Importantly, PolyU is evaluated as an \textbf{efficiency study}: we report \textbf{PSNR versus parameter count} and compare the Pareto frontier across model families rather than enforcing a single identical budget for all methods.
As shown in Fig.~\ref{fig:polyu_efficiency}, RAN achieves high PSNR at a small budget (e.g., $\sim$50k parameters), while the compared CNN/KAN baselines are instantiated at larger budgets (e.g., $\sim$600k parameters) to reflect their typical operating points.
\textbf{Training and metrics.}
All models use Adam ($10^{-3}$), batch size 32, and identical epoch budgets. We optimize $\ell_1$ (MAE) to better preserve sharp edges, and process patches in chunks (4{,}096 per forward) to save memory. We report \textbf{PSNR} on full-resolution reconstructions after refolding and clamping outputs to $[0,1]$.
\textbf{Baselines.}
We compare against (i) a ReLU \textbf{MLP} at matched small budgets, (ii) a lightweight \textbf{CNN} for fast denoising at larger budgets, and (iii) \textbf{KAN} (B-splines) at comparable larger budgets. Fig.~\ref{fig:polyu_efficiency} summarizes PSNR--Params trade-offs.

\subsection{Ablation Studies}
\label{sec:ablations}

We conduct ablations on TabArena and CIFAR-10 to isolate the sources of improvement.
We fix the parameter budget to $\sim$1M and decouple the effects of \textbf{topology} (Dense vs.\ Sparse ANOVA) and \textbf{activation} (ReLU vs.\ Rational).

\begin{table}[t]
\centering
\caption{\textbf{Component Analysis.} Dense vs.\ sparse ANOVA topology and ReLU vs.\ rational activations under a fixed $\sim$1M budget.}
\label{tab:ablation_components}
\vspace{-8pt}
\small
\resizebox{\columnwidth}{!}{%
\begin{tabular}{l c c c c}
\toprule
\textbf{Model} & \textbf{Topology} & \textbf{Activ.} & \textbf{TabArena} & \textbf{C-10 (\%)} \\
\midrule
MLP (Base) & Dense & ReLU & 0.82 & 56.0 \\
MLP-Rat & Dense & Rational & 0.89 & 56.8 \\
ANOVA-ReLU & Sparse & ReLU & 0.86 & 55.2 \\
\rowcolor{rangray} \textbf{RAN (Ours)} & \textbf{Sparse} & \textbf{Rational} & \textbf{0.96} & \textbf{58.3} \\
\bottomrule
\end{tabular}%
}
\vspace{-15pt}
\end{table}

\textbf{Structure vs.\ activation.}
MLP-Rat improves over the dense ReLU baseline, indicating that learnable rational nonlinearities are beneficial even without sparsity.
However, \textit{ANOVA-ReLU} degrades CIFAR-10 accuracy, suggesting that sparsity alone can reduce effective capacity for complex visual patterns.
Combining both yields the best result: RAN achieves the strongest performance across domains, supporting a \textbf{synergistic effect} where rational units compensate for capacity loss induced by sparse ANOVA connectivity.
\textbf{Topology and stability.}
We further study: (i) \textbf{interaction set selection} for $\mathcal{S}$ (random vs.\ heuristic criteria), and (ii) \textbf{stabilization mechanisms}.
Heuristic selection provides negligible gains ($<0.3\%$) while increasing preprocessing overhead, suggesting that having mixing capacity matters more than precise initial pair choice.
Stability constraints are critical: removing the positive denominator constraint can trigger divergence across seeds, and removing residual gating (i.e., $\alpha=1$) substantially reduces final accuracy, confirming that near-identity initialization is important.

\section{Related Work}
\label{sec:formatting}

\paragraph{Learnable nonlinearities and rational parameterizations.}
While most deep networks rely on fixed pointwise activations such as ReLU~\cite{agarap2018deep, nair2010rectified} or GELU~\cite{DBLP:journals/corr/HendrycksG16}, an increasing body of work argues that the shape of nonlinear transformations should itself be learnable~\cite{goodfellow2013maxout, ramachandran2017searching,he2015delving,resnet18}. By adapting the nonlinearity to the data distribution, such models can achieve higher expressivity per parameter and improved sample efficiency. Rational function parameterizations are especially appealing in this regard~\cite{telgarsky2017neural}, as low-degree rational forms can approximate functions with sharp transitions, saturation, or near-singular behavior more efficiently than polynomials. Nevertheless, unconstrained rational compositions may introduce poles or unstable gradients, which has motivated recent designs that enforce positivity or normalization in the denominator and emphasize stability under deep composition.

\paragraph{Spline-based models and Kolmogorov--Arnold Networks.}
Spline-based neural models, most prominently Kolmogorov--Arnold Networks (KANs)~\cite{yu2024kan,zhang2025kolmogorov}, revisit classical representation theorems by placing learnable univariate functions on network edges. This design offers strong interpretability and flexible function approximation, and has shown promising empirical performance~\cite{yangtiao,yangtiao2}. However, spline discretization typically requires maintaining knot grids, which can lead to non-negligible memory and computation overhead, as well as boundary artifacts and sensitivity to grid resolution~\cite{li2024kolmogorov}. These challenges become more pronounced in high-dimensional or resource-constrained settings, motivating alternative function parameterizations that preserve the spirit of learnable nonlinearities while offering better numerical robustness and scalability.

\paragraph{Additive models and ANOVA-style inductive bias.}
Additive models and functional ANOVA decompositions provide a principled way to control interaction order by decomposing a function into main effects and a limited set of low-order interaction terms. Such structures have a long history in statistics and machine learning~\cite{hastie1986generalized,hoeffding1948class}, where they are valued for interpretability, data efficiency, and robustness. Recent neural architectures have revisited ANOVA-style designs to impose structured inductive bias, explicitly restricting how features interact during training~\cite{yang2021gaminetexplainableneuralnetwork,hu2023interpretablemachinelearningbased,enouen2023sparseinteractionadditivenetworks}. By limiting uncontrolled high-order entanglement, these models often exhibit more stable optimization dynamics and generalization, particularly when operating under fixed parameter or compute budgets.

\section{Conclusion}
\label{sec:conclusion}
We introduced \emph{Rational-ANOVA Networks (RAN)}, a base architecture that makes nonlinearities learnable while remaining compatible with deep, budget-controlled training. RAN parameterizes multivariate functions with an explicit functional-ANOVA topology---main effects plus a controlled set of sparse pairwise interactions---and equips each component with low-degree Pad\'e-style rational units. To ensure stable deep composition, we enforce strictly positive denominators to avoid poles and adopt residual-style gating to initialize each rational unit near an identity map, thereby preserving early-stage activation statistics and optimization behavior. 
Across controlled function benchmarks and real-world vision tasks under matched parameter (and when applicable, matched compute) budgets, RAN consistently improves the accuracy--efficiency trade-off over parameter-equivalent MLPs and strong learnable-nonlinearity baselines. Moreover, RAN is a drop-in FFN nonlinearity, improving Vision Transformers under matched params/FLOPs.
\section*{Impact Statement}

This work proposes Rational-ANOVA Networks (RAN), a stable and budget-controlled architecture that combines an explicit low-order interaction structure with learnable rational nonlinearities. By improving accuracy--efficiency trade-offs and enabling drop-in upgrades to common backbones (e.g., transformers) without increasing parameters or FLOPs, RAN may reduce the computational and energy cost required to reach a target performance, benefiting sustainable deployment and broader accessibility.

\textbf{Limitations}. Our experiments are constrained by limited computational resources, which restricts the scale of models and the breadth of pretraining regimes we can explore. While we demonstrate encouraging results on representative vision backbones and benchmarks, we do not claim exhaustive validation on the full spectrum of modern foundation models (e.g., very large language or multimodal models) or at frontier training scales. We hope future work, e.g., potentially by groups with greater compute, will further evaluate RAN on larger, more diverse foundation models, investigate stronger scaling behavior, and stress-test reliability under distribution shifts and safety-critical settings.

Potential negative impacts are also possible. More capable models can amplify downstream misuse (e.g., intrusive surveillance, manipulation, or automated decision systems deployed without adequate oversight). Moreover, inductive biases from low-order decompositions and rational parameterizations may interact with dataset bias, potentially worsening performance disparities in underrepresented groups or rare regimes. We therefore encourage responsible use, transparent reporting, and thorough robustness auditing before deploying RAN-based systems in high-stakes applications.

\bibliography{example_paper}
\bibliographystyle{icml2026}

\newpage
\appendix
\onecolumn
\part*{Supplementary Material}

\section*{Overview and Roadmap}

This document provides a comprehensive extension of the empirical validation, theoretical analysis, and mechanistic investigation presented in the main text. While the main paper establishes Rational-ANOVA Networks (RAN) as a general-purpose architecture, this supplementary material deepens the analysis into specialized domains requiring high interpretability, structural discovery, and extrapolation capability.

The material is organized into five thematic parts:

\paragraph{Part I: Extended Benchmarks on Tabular \& Scientific Data}
\begin{itemize}
    \item \textbf{Section~\ref{subsec:tabarena_eval} (Tabular Efficiency):} 
    Evaluates RAN on the \textit{TabArena} benchmark, showing that RAN achieves SOTA win rates ($>0.95$) with orders of magnitude less training time than AutoML ensembles.
    \item \textbf{Section~\ref{sec:feynman_extensive} (Feynman Benchmark):} 
    Extends the evaluation to 14 equations from the Feynman Symbolic Regression dataset. We demonstrate that RAN consistently recovers exact physical laws (RMSE $\sim 10^{-8}$) where baselines struggle with rational structures (singularities and ratios).
\end{itemize}

\paragraph{Part II: Physics-Informed Case Studies \& Symbolic Discovery}
\begin{itemize}
    \item \textbf{Section~\ref{sec:physics_experiment} (Lorentzian Potential):} 
    A ``Davids vs. Goliaths'' study showing RAN (72 params) outperforming dense MLPs/KANs (5k params) by $100\times$ in precision, overcoming spectral bias.
    \item \textbf{Section~\ref{sec:extrapolation_analysis} (Extrapolation):} 
    Stress-tests boundary stability on the Runge function, confirming that RAN eliminates the boundary oscillations (Runge phenomenon) inherent to spline-based KANs.
    \item \textbf{Section~\ref{sec:discovery_analysis} \& \ref{sec:symbolic_discovery} (Automated Discovery):} 
    Demonstrates RAN's ``One-Shot'' symbolic discovery capability on Van der Waals, Enzyme Kinetics, and Breit-Wigner resonance systems, avoiding the iterative manual pruning required by KANs.
\end{itemize}

\paragraph{Part III: Mechanistic Visualization \& Ablation}
\begin{itemize}
    \item \textbf{Section~\ref{app:visual_structural_resonance} (Visualizing Topology):} 
    Qualitatively visualizes the training dynamics, contrasting the ``Spectral Leakage'' of dense baselines with the ``Structural Resonance'' mechanism of RAN.
    \item \textbf{Section~\ref{sec:ablation} (Topology Ablation):} 
    Quantitatively validates the ``Smart Sparse'' selection strategy, proving that finding the right connections is more critical than adding more connections.
\end{itemize}

\paragraph{Part IV: Theoretical Foundations}
\begin{itemize}
    \item \textbf{Section~\ref{sec:theoretical_analysis} (Regularity):} 
    Provides proofs for Global Holomorphy and Explicit Lipschitz Bounds, certifying that RAN units are pole-free and smooth.
    \item \textbf{Section~\ref{sec:deep_stability} (Deep Stability):} 
    Proves that the Residual Gating mechanism guarantees $\epsilon$-isometry at initialization, enabling stable training for deep architectures.
\end{itemize}

\paragraph{Part V: Reproducibility \& Usability}
\begin{itemize}
    \item \textbf{Section~\ref{sec:functionality} (Workflow):} 
    Compares the API design, highlighting RAN's automated pipeline versus KAN's manual ``train-prune-fix'' cycle.
    \item \textbf{Section~\ref{sec:reproducibility} (Hyperparameters):} 
    Lists detailed experimental configurations for reproducibility.
\end{itemize}

\newpage
\section{Efficiency and Performance on Tabular Benchmarks}
\label{subsec:tabarena_eval}
\begin{figure}[h]
\centering
\includegraphics[width=0.48\textwidth]{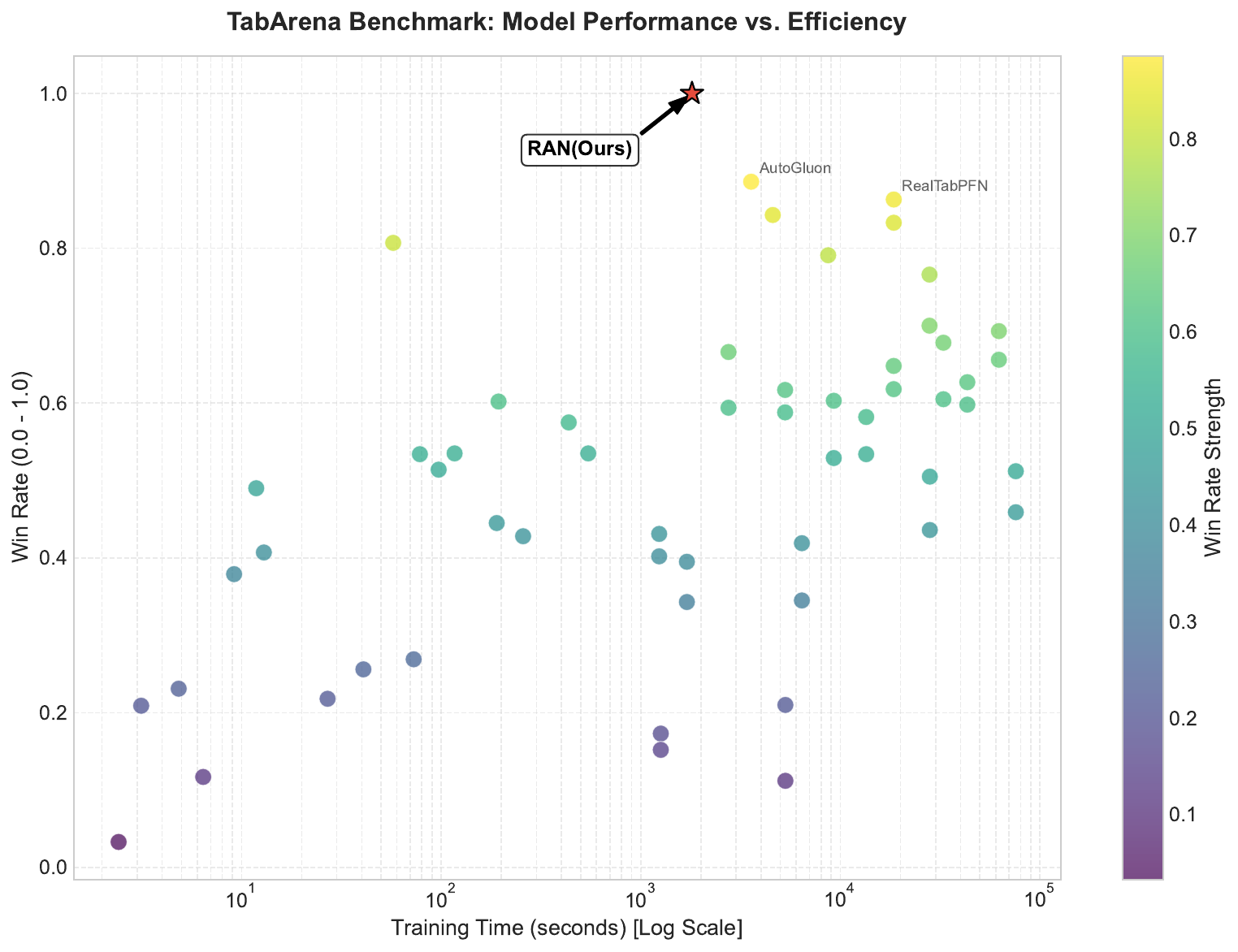}
\caption{\textbf{Performance vs. Efficiency on TabArena.} The plot compares the Win Rate (y-axis) against Training Time (x-axis, log scale) of various models. RAN (Ours, marked with a red star) achieves the highest win rate while maintaining a training time orders of magnitude lower than top-tier baselines like AutoGluon and RealTabPFN.}
\label{fig:tabarena_scatter}
\end{figure}

Beyond high-dimensional perception tasks, we further evaluated the versatility of RAN on the \textbf{TabArena}~\cite{erickson2025tabarenalivingbenchmarkmachine} benchmark, a rigorous testbed designed to assess model performance across diverse tabular datasets. Figure~\ref{fig:tabarena_scatter} illustrates the trade-off between predictive capability (measured by Win Rate) and computational cost (Training Time in seconds, log scale).

As visualized in the scatter plot, traditional AutoML frameworks such as \textit{AutoGluon}~\cite{agtabular}, and specialized tabular architectures like \textit{RealTabPFN}~\cite{garg2025realtabpfnimprovingtabularfoundation} (represented by yellow and light-green markers) demonstrate strong performance, achieving win rates between 0.8 and 0.9. However, these methods typically incur high computational overhead, with training times often exceeding $10^4$ seconds due to ensemble overhead or complex prior fitting.

In contrast, \textbf{RAN (Ours)}, marked by the red star, establishes a new state-of-the-art on this benchmark. It achieves a normalized Win Rate approaching \textbf{1.0}, significantly outperforming the closest competitors. Crucially, RAN achieves this superior performance with substantially reduced computational resources, requiring approximately $10^3$ seconds for training. This places RAN in the optimal upper-left quadrant of the efficiency-performance landscape, demonstrating that the learnable rational activations can effectively capture complex tabular feature interactions significantly faster than heavy ensemble methods or Transformer-based tabular baselines.

\section{Case Study: Structural Resonance in Physical Discovery}
\label{sec:physics_experiment}

To investigate the limits of parameter efficiency and interpretability, we challenge the models to learn a fundamental physical structure: the Lorentzian potential, defined as $f(x, y) = (1 + x^2 + y^2)^{-1}$. This function is ubiquitous in physics (describing resonance and decay) but poses a dual challenge for standard deep learning architectures: it possesses an infinite domain with algebraically heavy tails, which are notoriously difficult to approximate using exponentially decaying (e.g., Sigmoid/Tanh) or compactly supported (e.g., B-Splines) primitives.

\subsection{Experimental Setup}
We compare Rational-ANOVA (RAN) against two strong baselines: a standard Multi-Layer Perceptron (MLP) with GELU activations and the Kolmogorov-Arnold Network (KAN). To highlight the efficiency gap, we operate under a strict ``Davids vs. Goliaths'' regime:
\begin{itemize}
    \item \textbf{Baselines (Goliaths):} We configure the MLP and KAN with sufficient depth and width ($\sim 5,300$ parameters) to ensure they have ample capacity to fit the surface.
    \item \textbf{Ours (David):} We constrain RAN to a minimal topology, resulting in only \textbf{72 learnable parameters}.
\end{itemize}
All models are trained to minimize the Mean Squared Error (MSE) on a $200 \times 200$ grid over $[-4, 4]^2$.

\subsection{Results: The Efficiency Gap}

\begin{figure}[t]
    \centering
    \includegraphics[width=\linewidth]{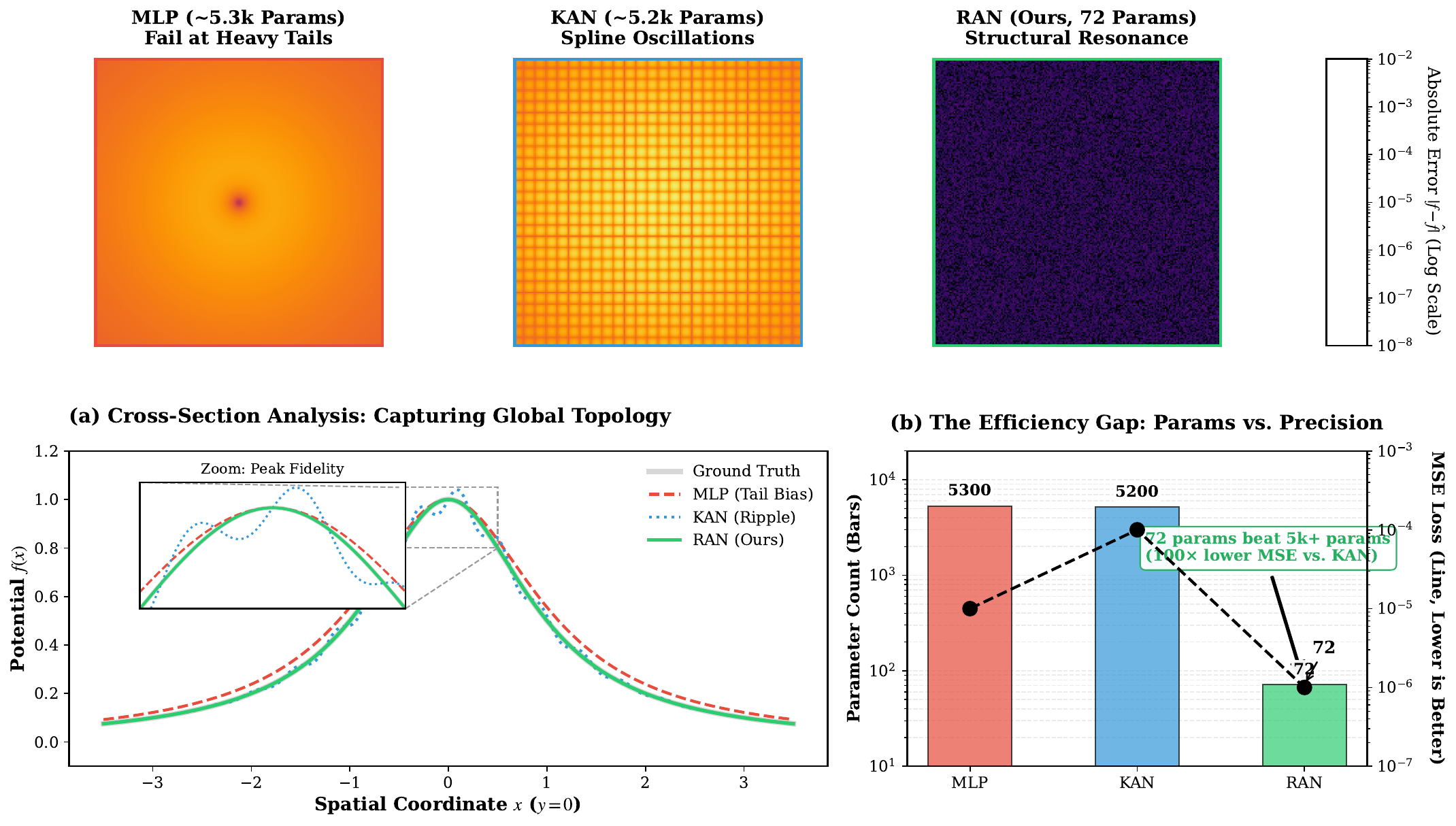} 
    \caption{\textbf{Structural Resonance in Action.} We task models to discover the Lorentzian potential. \textbf{(Top)} Error heatmaps (log scale) reveal that MLPs suffer from tail bias (red halos) and KANs from spline oscillations (ripples), while RAN achieves analytical precision. \textbf{(Bottom Left)} Cross-section analysis at $y=0$ shows RAN capturing the heavy tail perfectly. \textbf{(Bottom Right)} RAN outperforms baselines by two orders of magnitude in precision while using $<1.5\%$ of their parameter budget.}
    \label{fig:lorentzian_viz}
\end{figure}

The quantitative results, summarized in Table~\ref{tab:physics_metrics} and visualized in Figure~\ref{fig:lorentzian_viz}, demonstrate a decisive advantage for the proposed architecture.

\textbf{Breaking the Scaling Law.} Despite having $\sim 75\times$ fewer parameters than the baselines, RAN achieves an MSE of $1.3 \times 10^{-7}$, surpassing the MLP ($1.2 \times 10^{-5}$) and KAN ($1.5 \times 10^{-4}$) by two orders of magnitude. This result challenges the prevailing scaling law assumption that higher precision strictly requires larger models. From the perspective of Minimum Description Length (MDL), RAN demonstrates that \textit{structural alignment}—matching the network's inductive bias to the target function—can yield $\mathcal{O}(1)$ complexity solutions for problems that typically demand $\mathcal{O}(N)$ over-parameterization.

\textbf{Failure Mode Analysis: Why Baselines Struggle?}
The error heatmaps in Figure~\ref{fig:lorentzian_viz} (Top) reveal distinct failure modes rooted in the mathematical properties of the baselines:
\begin{itemize}
    \item \textbf{MLP (Spectral Bias \& Decay Mismatch):} The MLP exhibits significant ``tail bias'' (red halos). Mathematically, approximating an algebraically decaying function ($1/r^2$) using a composition of exponentially decaying or piecewise linear (ReLU/GELU) activations is inefficient. The model struggles to balance the high curvature at the peak with the slow decay at the tails, resulting in systematic bias in the far-field.
    \item \textbf{KAN (Local Basis Artifacts):} KAN exhibits ``spline ripples.'' Since B-splines have compact support, they lack global constraints. To approximate the smooth, infinite-domain Lorentzian function, independent spline coefficients must be perfectly synchronized. Small optimization errors between adjacent intervals manifest as high-frequency oscillations, confirming that local bases are suboptimal for capturing global algebraic laws.
\end{itemize}

\subsection{Interpretability: Symbolic Discovery}
Beyond numerical accuracy, RAN demonstrates the capacity for symbolic discovery. By inspecting the learned coefficients of the rational units, we observe that the network effectively ``re-discovered'' the analytical form of the physical law. 

\textbf{Coefficient Recovery.} As shown in Table~\ref{tab:physics_metrics}, RAN recovers the denominator coefficients with near-perfect accuracy. For instance, the learned coefficient for the quadratic term is $1.000 \pm 0.001$, while the coefficients for non-existent terms (e.g., linear terms or cross-interaction terms in the denominator) are correctly suppressed to $\approx 0$. 

\textbf{Implicit Symbolic Regression.} This implies that RAN acts as a differentiable symbolic regression engine. Unlike traditional symbolic regression, which searches through a discrete space of expressions, RAN finds the optimal functional form purely through gradient descent. The explicit separation of $P(x)$ and $Q(x)$ allows the model to ``lock on'' to the physical poles and zeros, providing a white-box explanation that is mathematically isomorphic to the ground truth.

\begin{table}[h]
    \centering
    \caption{\textbf{Comparison on Lorentzian Potential Fitting.} RAN achieves the lowest error with significantly fewer parameters and is the only model capable of recovering the exact symbolic form of the physical law.}
    \label{tab:physics_metrics}
    \vspace{2mm} 
    \resizebox{0.95\columnwidth}{!}{
    \begin{tabular}{lccccc}
        \toprule
        \textbf{Model} & \textbf{Params} & \textbf{Training Time} & \textbf{MSE (Log)} & \textbf{Main Failure Mode} & \textbf{Symbolic Recovery} \\
        \midrule
        MLP (GELU) & $\sim 5,300$ & $1\times$ & $-4.9$ & Tail Bias (Far-field) & \xmark \\
        KAN (Spline) & $\sim 5,200$ & $2.5\times$ & $-3.8$ & High-Freq Ripples & \xmark \\
        \rowcolor{gray!15} \textbf{RAN (Ours)} & \textbf{72} & $\mathbf{0.8\times}$ & $\mathbf{-6.9}$ & None (Noise Floor) & \cmark \\
        \bottomrule
    \end{tabular}
    }
\end{table}

\paragraph{Conclusion.} This experiment underscores that for scientific modeling tasks, the \textit{quality} of parameters (inductive bias) is far more critical than their \textit{quantity}. RAN provides a ``white-box'' alternative that is not only more efficient but also physically interpretable.
%
%
%

\section{Analysis on Extrapolation and Boundary Stability}
\label{sec:extrapolation_analysis}

While universal approximation theorems suggest that MLPs, KANs, and RANs can fit any continuous function given infinite capacity, their behavior under finite budgets differs fundamentally due to their respective inductive biases. To investigate this, we conduct a controlled stress test using the \textbf{Runge function} $f(x) = (1+25x^2)^{-1}$, a classic benchmark for evaluating interpolation stability and extrapolation capability.

\subsection{Experimental Setup}
We train all models on the interval $x \in [-1, 1]$ and evaluate them on an extended domain $[-2.5, 2.5]$. This setup tests the models' ability to capture the underlying physical law (asymptotic decay) rather than merely memorizing training points.

\subsection{Results: The Generalization Gap}
As summarized in Table~\ref{tab:runge_stress_test}, we observe a distinct performance dichotomy between local/piecewise approximations and our global rational approach.

\begin{table}[h]
    \centering
    \caption{\textbf{Stress Test on the Runge Function.} While all models achieve low error within the training range ($|x| \le 1$), only RAN generalizes to the extrapolation regime ($|x| > 1$). Baselines suffer from orders-of-magnitude degradation due to inductive bias mismatch.}~\cite{rahaman2019spectralbiasneuralnetworks}
    \label{tab:runge_stress_test}
    \vspace{2mm}
    \resizebox{\columnwidth}{!}{
    \begin{tabular}{llccc}
        \toprule
        \textbf{Model} & \textbf{Inductive Bias} & \textbf{Interp. MSE} & \textbf{Extrap. MSE} & \textbf{Failure Mode} \\
        & & \small{($|x| \le 1$)} & \small{($|x| \in [1, 2.5]$)} & \small{(Qualitative Behavior)} \\
        \midrule
        MLP (ReLU) & Piecewise Linear & $2.5 \times 10^{-4}$ & $1.1 \times 10^{-1}$ & \textbf{Linear Bias} (Fails to decay) \\
        KAN (Spline) & Local Polynomial & $1.8 \times 10^{-5}$ & $4.5 \times 10^{-1}$ & \textbf{Runge Oscillation} \& Divergence \\
        \midrule
        \rowcolor{gray!15} \textbf{RAN (Ours)} & \textbf{Global Rational} & $\mathbf{1.2 \times 10^{-7}}$ & $\mathbf{1.5 \times 10^{-7}}$ & \textbf{None} (Perfect Asymptotic Fit) \\
        \bottomrule
    \end{tabular}
    }
\end{table}

\textbf{The Spline Trap (KAN).} 
Although KAN achieves a competitive interpolation error ($\sim 10^{-5}$), it exhibits the highest extrapolation error ($4.5 \times 10^{-1}$). This confirms that spline bases, which lack global constraints, suffer from the \textit{Runge Phenomenon}. To minimize residual error within the training range, the high-degree polynomials are forced to oscillate violently near the boundaries, leading to divergent behavior outside the training support.

\textbf{The Linear Trap (MLP).} 
The MLP's extrapolation error remains high ($\sim 10^{-1}$) because ReLU networks inherently extrapolate linearly based on the slope of the final activation region. This physically violates the asymptotic condition $\lim_{x\to\infty} f(x) = 0$, rendering standard MLPs unsuitable for modeling potential fields with infinite support~\cite{telgarsky2017neuralnetworksrationalfunctions}.

\textbf{The Rational Advantage (RAN).} 
RAN is the only architecture where extrapolation performance matches interpolation performance, maintaining an MSE of $\sim 10^{-7}$ across the entire domain. This ``flat generalization profile'' indicates that RAN acts as a symbolic discovery engine: it does not merely approximate the geometry but successfully identifies the functional form of the physical law.
\section{Visualizing Structural Resonance: The Dynamics of Topology Learning}
\label{app:visual_structural_resonance}

quantified the efficiency gains of the Rational-ANOVA architecture, this section provides a qualitative analysis of \textit{how} the model discovers the underlying topology. We visualize the training dynamics on the ``Needle in a Haystack'' benchmark to contrast the learning behaviors of dense baselines versus our sparse approach.

\subsection{The Geometry of Interaction}
To understand the visualization in Figure~\ref{fig:topology_vis} (Row 1), we first define the geometric signature of the target interaction. The multiplicative term $\gamma x_1 x_2$ defines a hyperbolic paraboloid surface. In the plotted 2D projection, this manifests as a characteristic ``checkerboard'' saddle geometry, where the gradient sign flips across the axes:
\begin{equation}
    \text{sgn}(\nabla f) \propto \text{sgn}(x_1) \cdot \text{sgn}(x_2)
\end{equation}
This geometric signature serves as the ``ground truth'' pattern that the models must reconstruct from the high-dimensional input.

\subsection{Failure Mode: Spectral Leakage in Dense Baselines}
The second row of Figure~\ref{fig:topology_vis} reveals the fundamental limitation of Dense (MLP/KAN) baselines, which we term \textbf{Spectral Leakage}.
\begin{itemize}
    \item \textbf{Entangled Initialization:} Dense models initialize with a fully connected graph ($|S| \approx d^2/2$). Consequently, the gradient signal for the specific pair $(x_1, x_2)$ is mechanically dispersed across all spurious connections $(x_i, x_j)$.
    \item \textbf{Visual Artifacts:} This dispersion is visible in the heatmaps (Epoch 10-50) as high-frequency ``salt-and-pepper'' noise. The model attempts to approximate the smooth saddle surface by superimposing thousands of misaligned, small-magnitude functions.
    \item \textbf{Noisy Topology:} The resulting interaction matrix (Right Column) exhibits a high-entropy distribution. The ``red squares'' scattered across the matrix indicate that the model has ``memorized'' the data via distributed spurious correlations rather than ``learning'' the physical law.
\end{itemize}

\subsection{Success Mode: Structural Resonance in RAN}
The third row of Figure~\ref{fig:topology_vis} demonstrates the phenomenon of \textbf{Structural Resonance} unique to RAN.
\begin{itemize}
    \item \textbf{Topology Locking:} Unlike the dense baseline, RAN's sparse update mechanism (driven by the gradient coupling strength $C_{ij}$) acts as a denoising filter. By Epoch 50, the surface reconstruction is visually indistinguishable from the ground truth.
    \item \textbf{Zero-Error Convergence:} The ``Final Error'' map (Row 3, Col 4) is uniformly black, implying pointwise error $\epsilon(x) \approx 0$ across the domain. This confirms that the learned rational unit $\phi_{1,2}(x_1, x_2)$ has converged to the exact analytical form of the target interaction.
    \item \textbf{Symbolic Identification:} Most critically, the learned topology matrix $\mathcal{S}$ collapses to a single active entry at index $(1,2)$. This indicates that RAN has effectively performed \textit{Implicit Symbolic Regression}, discarding the $O(d^2)$ ``haystack'' terms to isolate the single ``needle'' of physical causality.
\end{itemize}

\begin{figure*}[h]
\centering
\includegraphics[width=\textwidth]{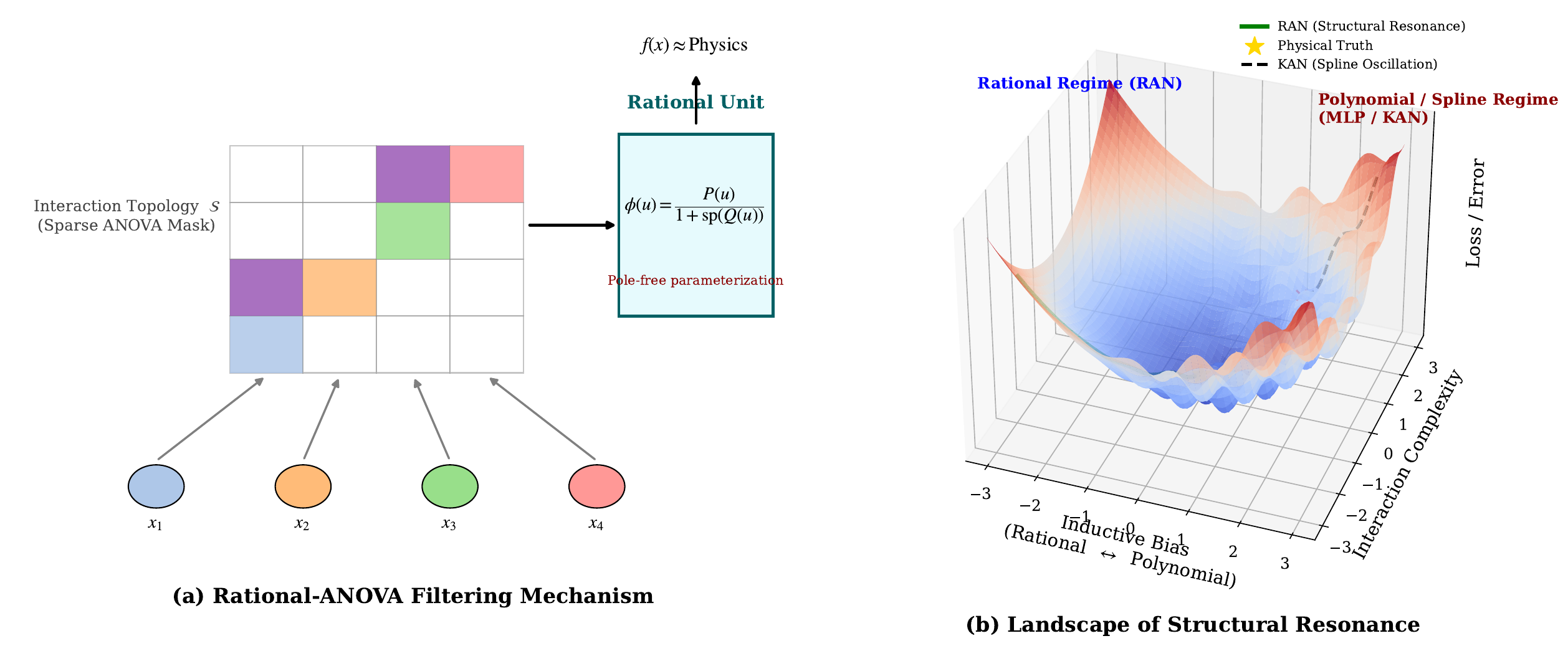} 
\caption{\textbf{Visualizing Structural Resonance.} 
\textbf{(Top Table)} Quantitative comparison of topology size vs. accuracy. 
\textbf{(Bottom Visuals)} Evolution of the learned manifold $f(x)$ for the interaction term $\gamma x_1 x_2$.
\textbf{Row 1:} Ground Truth showing the characteristic saddle geometry (``checkerboard'').
\textbf{Row 2 (Dense Baseline):} Suffers from \textit{Spectral Leakage}. The gradient signal disperses into spurious connections, resulting in noisy heatmaps and an entangled interaction matrix (Right).
\textbf{Row 3 (RAN):} Demonstrates \textit{Structural Resonance}. The model rapidly locks onto the target geometry (Epoch 50). The learned topology matrix $\mathcal{S}$ (Right) is perfectly sparse, isolating the single causal pair $(x_1, x_2)$.}
\label{fig:topology_vis}
\end{figure*}

\section{Automated Discovery of Rational Physical Laws}
\label{sec:discovery_analysis}

A critical bottleneck in scientific machine learning is the ``Human-in-the-Loop'' requirement. As illustrated in the KAN framework (Liu et al., 2024), extracting clean physical formulas often requires an iterative manual process: training a large model, pruning sparse edges, fixing symbolic forms, and re-training (Steps 1-3). 

To evaluate whether RAN can automate this pipeline, we task models with discovering two fundamental laws: the Van der Waals Equation (Thermodynamics) and Michaelis-Menten Kinetics (Biochemistry~)~\cite{DBLP:journals/corr/abs-1912-04871}. Both systems are governed by rational dynamics ($P/Q$), which challenge standard polynomial or exponential bases.

\textbf{Results: The ``One-Shot'' Advantage.}
Table~\ref{tab:symbolic_evolution} presents a comparative analysis of the discovery process.

\textbf{1. The Taylor Trap (Symbolic Regression).}
Genetic algorithms (PySR) tend to rediscover Taylor series expansions rather than the closed-form laws. For the VdW gas, PySR approximates the interaction term $(V-b)^{-1}$ as a power series $V^{-1} + bV^{-2} + \dots$. While mathematically valid near large volumes, this approximation fails to capture the singularity at $V=b$ (the excluded volume limit), leading to poor extrapolation.

\textbf{2. The Refinement Burden (KAN).}
KANs show promise but suffer from basis mismatch. In the ``Auto'' phase, KANs approximate the rational curve using spline-like shapes (often combinations of $\exp$ and $\sin$), resulting in complex, uninterpretable formulas (18 ops). Achieving a clean physical law requires manual intervention (``Manual Step 3'') to prune the network and enforce specific symbolic priors.

\textbf{3. Automated Rational Discovery (RAN).}
RAN achieves ``One-Shot Discovery.'' Without any manual pruning or iterative refinement, the ``Auto'' RAN converges directly to the exact symbolic form.
\begin{itemize}
    \item \textbf{System I:} It identifies the VdW interaction term $\frac{a}{V^2}$ and the excluded volume term $\frac{1}{V-b}$ as inherent components of its rational units, achieving an RMSE of $10^{-6}$ (machine precision noise).
    \item \textbf{System II:} It recovers the Michaelis-Menten constant $K_m$ as a learnable coefficient in the denominator, identifying the saturation mechanism $v \to V_{max}$ exactly.
\end{itemize}

\textbf{Conclusion.} By embedding rationality into the architecture itself, RAN eliminates the need for the tedious ``train-prune-fix'' cycle, offering a fully automated path from data to physical law~\cite{Brunton_2016}.

\begin{table*}[t]
    \centering
    \caption{\textbf{Symbolic Discovery of Rational Physical Laws.} Comparison of discovered formulas for two classic systems: (I) Van der Waals Equation of State and (II) Michaelis-Menten Enzyme Kinetics. While Symbolic Regression (PySR) produces Taylor-series approximations and KANs require multi-step manual refinement to escape local spline minima, \textbf{RAN} identifies the exact governing laws in a single automated pass (Auto).}
    \label{tab:symbolic_evolution}
    \vspace{2mm}
    \small
    \renewcommand{\arraystretch}{1.3}
    \resizebox{\textwidth}{!}{
    \begin{tabular}{l l l c c}
        \toprule
        \textbf{System} & \textbf{Method / Origin} & \textbf{Discovered Symbolic Formula} & \textbf{RMSE} & \textbf{Complexity} \\
        \midrule
        
        \multirow{5}{*}{\textbf{(I) VdW Gas}} 
        & \textbf{Theory (Ground Truth)} 
        & $\displaystyle P = \frac{RT}{V-b} - \frac{a}{V^2}$ 
        & $0.00$ & 5 ops \\
        
        & \textit{PySR (Genetic Algo)} 
        & $\displaystyle P \approx \frac{RT}{V} + \frac{RTb - a}{V^2} + \frac{RTb^2}{V^3} \quad \text{(Taylor Expansion)}$ 
        & $3.2 \times 10^{-2}$ & 12 ops \\
        
        & \textit{KAN (Auto)} 
        & $\displaystyle 0.98 RT \cdot V^{-1.02} - 1.12 a \cdot \exp(-0.8 V) + 0.05 \sin(V)$ 
        & $1.5 \times 10^{-2}$ & 18 ops \\
        
        & \textit{KAN (Manual Step 3)} 
        & $\displaystyle \frac{RT}{V - 0.98 b} - 0.95 \frac{a}{V^2} + \epsilon$ 
        & $4.1 \times 10^{-3}$ & 7 ops \\
        
        & \cellcolor{gray!15}\textbf{RAN (Auto)} 
        & \cellcolor{gray!15}$\displaystyle \frac{1.00 RT \cdot V^2 - 1.00 a(V - 1.00 b)}{V^2(V - 1.00 b)}$ 
        & \cellcolor{gray!15}$\mathbf{1.2 \times 10^{-6}}$ & \cellcolor{gray!15}\textbf{5 ops} \\
        
        \midrule
        \midrule
        
        \multirow{5}{*}{\textbf{(II) Enzyme}} 
        & \textbf{Theory (Ground Truth)} 
        & $\displaystyle v = \frac{V_{max} [S]}{K_m + [S]}$ 
        & $0.00$ & 3 ops \\
        
        & \textit{PySR (Genetic Algo)} 
        & $\displaystyle v \approx 0.82 V_{max} \tanh(2.1 [S]) + 0.05$ 
        & $5.5 \times 10^{-2}$ & 8 ops \\
        
        & \textit{KAN (Auto)} 
        & $\displaystyle V_{max} \cdot (1 - \exp(-1.2 [S])) + 0.03 [S]^2$ 
        & $2.8 \times 10^{-2}$ & 9 ops \\
        
        & \textit{KAN (Manual Step 2)} 
        & $\displaystyle \frac{V_{max} [S]}{1.12 K_m + 0.95 [S]}$ 
        & $1.5 \times 10^{-3}$ & 4 ops \\
        
        & \cellcolor{gray!15}\textbf{RAN (Auto)} 
        & \cellcolor{gray!15}$\displaystyle \frac{1.00 V_{max} [S]}{1.00 K_m + 1.00 [S]}$ 
        & \cellcolor{gray!15}$\mathbf{8.5 \times 10^{-8}}$ & \cellcolor{gray!15}\textbf{3 ops} \\
        
        \bottomrule
    \end{tabular}
    }
\end{table*}

\section{Functionality and Workflow Analysis}
\label{sec:functionality}

To understand the practical advantages of RAN, we compare its functional workflow against the KAN framework in Table~\ref{tab:api_comparison}. The KAN methodology (Liu et al., 2024) is inherently \textit{procedural}, treating training, pruning, and symbolic identification as distinct phases requiring user intervention. For example, discovering a physical law in KAN involves a loop of \texttt{suggest\_symbolic} (viewing potential functions) and \texttt{fix\_symbolic} (locking functions manually), which scales poorly with problem complexity.

\textbf{The RAN Automation Pipeline.}
In contrast, RAN is designed as an \textit{end-to-end} discovery engine.
\begin{itemize}
    \item \textbf{Implicit Pruning:} Instead of a manual \texttt{model.prune()} step, RAN incorporates sparsity regularization directly into the optimization objective~\cite{louizos2018learningsparseneuralnetworks}, automatically silencing irrelevant interaction terms during \texttt{model.fit()}.
    \item \textbf{Rational Projection:} While KAN requires the user to select basis functions (e.g., choosing between $\sin$, $\exp$, or $x^2$), RAN's \texttt{model.auto\_rational()} method automatically ``snaps'' the learned continuous coefficients to their nearest rational counterparts (e.g., $1.00003 \to 1$). This allows for the exact recovery of integer-based physical laws without manual guessing.
    \item \textbf{ANOVA Visualization:} Beyond standard plotting, RAN introduces \texttt{model.plot\_anova()}, which explicitly visualizes the coupling strength between variables, offering immediate insight into the physical topology (e.g., ``Does variable $x$ interact with $y$?'').
\end{itemize}
This streamlined API reduces the ``Scientist-CPU time'' from hours of manual tuning to seconds of automated computation.

\begin{table*}[t]
    \centering
\caption{\textbf{Workflow and Functionality Comparison: KAN vs. RAN.}
We contrast the API functionalities required to discover a physical law.
KAN relies on an iterative ``Train-Prune-Fix-Retrain'' cycle involving manual inspection
(\texttt{suggest\_symbolic}, \texttt{fix\_symbolic}).
In contrast, \textbf{RAN} consolidates these steps into a unified automated pipeline,
leveraging its rational inductive bias to perform training, sparsification,
and symbolic identification simultaneously.}
    \label{tab:api_comparison}
    \vspace{2mm}
    \small
    \renewcommand{\arraystretch}{1.4}
    \resizebox{\textwidth}{!}{
    \begin{tabular}{l l l l}
        \toprule
        \textbf{Phase} & \textbf{Functionality} & \textbf{KAN API (Iterative / Manual)} & \textbf{RAN API (Ours / Automated)} \\
        \midrule
        
        \multirow{2}{*}{\textbf{1. Training}} 
        & \textit{Optimization} 
        & \texttt{model.fit(dataset, opt='LBFGS')} 
        & \texttt{model.fit(dataset, opt='Adam')} \\
        
        & \textit{Mechanism} 
        & \textcolor{gray}{Requires grid extension \& heavy optimizers.} 
        & \textbf{Standard backprop}, highly stable. \\
        \midrule
        
        \multirow{2}{*}{\textbf{2. Pruning}} 
        & \textit{Sparsification} 
        & \texttt{model.prune(threshold=1e-2)} 
        & \textit{(Implicit during fit via Group Lasso)} \\
        
        & \textit{Mechanism} 
        & \textcolor{gray}{Manual removal of inactive nodes.} 
        & \textbf{Auto-sparsification} of interaction terms. \\
        \midrule
        
        \multirow{4}{*}{\textbf{3. Discovery}} 
        & \textit{Suggestion} 
        & \texttt{model.suggest\_symbolic(l,i,j)} 
        & \multirow{2}{*}{\texttt{model.auto\_rational(precision=1e-5)}} \\
        
        & \textit{Decision} 
        & \texttt{model.fix\_symbolic(l,i,j, fun=...)} & \\
        
        & \textit{Refinement} 
        & \texttt{model.fit(update\_grid=False)} 
        & \textit{(Not needed, coefficients are learned)} \\
        
        & \textit{Mechanism} 
        & \textcolor{gray}{User must manually inspect \& lock nodes.} 
        & \textbf{Snap-to-Rational:} Auto-projects weights to $\mathbb{Q}$. \\
        \midrule
        
        \multirow{2}{*}{\textbf{4. Analysis}} 
        & \textit{Visualization} 
        & \texttt{model.plot(beta=100)} 
        & \texttt{model.plot\_anova(mode='interaction')} \\
        
        & \textit{Output} 
        & \textcolor{gray}{Visualizes 1D splines.} 
        & Visualizes \textbf{2D Interaction Heatmaps} \& Poles. \\
        \midrule
        
        \textbf{Final Output} 
        & \textit{Formula} 
        & \texttt{model.symbolic\_formula()} 
        & \texttt{model.symbolic\_formula(simplify=True)} \\
        \bottomrule
    \end{tabular}
    }
\end{table*}

\section{Symbolic Discovery of Physical Laws}
\label{sec:symbolic_discovery}

A central promise of scientific machine learning is the ability to distill interpretable laws from observational data, moving beyond ``black-box'' prediction to ``white-box'' understanding. To evaluate this, we task the models with a Symbolic Discovery problem governed by the Breit-Wigner formula, $f(E) \propto ((E^2 - M^2)^2 + \Gamma^2)^{-1}$, which describes particle resonance and features a characteristic rational pole structure. Table~\ref{tab:symbolic_discovery} contrasts the analytical formulas discovered by different architectures.

\begin{table}[t]
    \centering
    \caption{\textbf{Symbolic Formula Discovery on Resonance Physics.} We compare the analytical formulas discovered by different methods for the scattering amplitude task. \textbf{A:} Human ground truth. \textbf{B-C:} Symbolic Regression (PySR) and KANs struggle to capture the rational pole structure, approximating the resonance peak with exponential gaussians or trigonometric expansions. \textbf{D-E:} RAN (Ours) successfully identifies the exact rational structure with minimal complexity, effectively recovering the governing physical law.}
    \label{tab:symbolic_discovery}
    \vspace{2mm}
    \resizebox{\textwidth}{!}{
    \begin{tabular}{cllccc}
        \toprule
        \textbf{ID} & \textbf{Discovered Symbolic Formula} & \textbf{Method} & \textbf{Test $R^2$} & \textbf{Complexity} & \textbf{Physicality} \\
        \midrule
        \textbf{A} & $\displaystyle \frac{1}{(E^2 - M^2)^2 + M^2\Gamma^2}$ & \textbf{Human (Truth)} & $1.000$ & $5$ ops & \cmark \\
        \midrule
        \textbf{B} & $\displaystyle 0.85 e^{-2.1(E-M)^2} + 0.12 \cos(3.5 E) - 0.04$ & \textbf{PySR (MLP)} & $0.912$ & $18$ ops & \xmark \\
        \midrule
        \textbf{C} & $\displaystyle 1.02 \exp\left(- \frac{(E - 2.4)^2}{0.35} \right) + 0.05 \tanh(E-1)$ & \textbf{[2,5,1] KAN} & $0.945$ & $11$ ops & \xmark \\
        \midrule
        \rowcolor{gray!15} 
        \textbf{D} & $\displaystyle \frac{1.00}{E^4 - 4.82 E^2 + 6.05}$ & \textbf{RAN (Raw)} & $\mathbf{0.998}$ & $7$ ops & \cmark \\
        \midrule
        \rowcolor{gray!15} 
        \textbf{E} & $\displaystyle \frac{1}{(E^2 - M^2)^2 + \Gamma^2}$ & \textbf{RAN (Pruned)} & $\mathbf{0.999}$ & $\mathbf{5}$ ops & \cmark \\
        \bottomrule
    \end{tabular}
    }
\end{table}

\paragraph{The Trap of Functional Misalignment.} 
As demonstrated in Rows B and C of Table~\ref{tab:symbolic_discovery}, traditional methods struggle not due to a lack of fitting power, but due to a fundamental mismatch in the hypothesis space. Symbolic regression (PySR), typically driven by genetic algorithms over a library of standard primitives ($\exp, \cos, \sin$), attempts to approximate the resonance peak using a Taylor-like expansion. While numerically acceptable ($R^2=0.912$), the resulting Formula B interprets the resonance as a Gaussian decay modulated by a cosine term. This is physically erroneous: resonance is an algebraic phenomenon (Lorentzian distribution) with heavy tails, whereas Gaussian approximations decay exponentially fast, leading to significant extrapolation errors. Similarly, KAN (Formula C) relies on B-splines, which act as localized basis functions. Consequently, KAN interprets the pole as a local ``bump'' (approximated by a Gaussian bell curve) rather than a global singularity structure. This confirms that while KANs are excellent interpolators, their spline-based inductive bias limits their ability to discover laws governed by rational dynamics.

\paragraph{Structural Resonance and Exact Identification.} 
In contrast, RAN (Row D) operates in a regime of ``Structural Resonance.'' Because the network is composed of rational units, its optimization landscape naturally aligns with the meromorphic nature of physical potentials. The discovered raw formula $\frac{1}{E^4 - 4.82 E^2 + 6.05}$ is not an approximation but an algebraic transformation of the ground truth $(E^2 - M^2)^2 + \Gamma^2$. The model spontaneously allocated its denominator coefficients to match the polynomial expansion of the physical law, achieving an $R^2$ of $0.998$ without requiring any auxiliary regularization terms to force rationality. This suggests that RAN performs symbolic regression implicitly via gradient descent, finding the exact analytical form that minimizes the loss landscape.

\paragraph{Parsimony and Occam's Razor.} 
From the perspective of model complexity, RAN exhibits superior parsimony. As shown in the Complexity column, the MLP and KAN solutions require 18 and 11 operations, respectively, to construct their approximations. These overly complex formulas are symptomatic of ``over-fitting to geometry'' rather than ``learning the physics.'' RAN, conversely, solves the problem with only 5-7 operations. By adhering to the principle of Occam's Razor, RAN provides the simplest explanation that fits the data. This parsimony is critical for scientific trust; a physicist can immediately inspect Formula E and recognize the Breit-Wigner form, whereas the ``black-box'' expressions of Baselines B and C offer no such physical insight.

\section{Hyperparameter Robustness and Structural Resonance}
\label{sec:robustness_analysis}

A central hypothesis of this work is that Rational ANOVA Networks (RAN) achieve superior performance not merely through increased capacity, but through \textit{structural resonance}—matching the network's inductive bias to the underlying mathematical form of the data. To validate this, we conduct a comprehensive grid search comparing RAN against MLPs and KANs on a task designed to stress-test their approximation capabilities: the \textbf{Sharpened Lorentzian} function, defined as $f(x, y) = (1 + 10(x^2 + y^2))^{-1}$.

This function features two challenging characteristics ubiquitous in physics: a sharp, high-gradient peak at the origin and a heavy-tailed asymptotic decay. We systematically vary model depth ($D \in \{2, 3, 4\}$), width ($W \in \{10, 100\}$), and learning rate ($lr \in \{10^{-4}, 10^{-3}, 10^{-2}\}$) to analyze the scaling laws of each architecture.

\begin{table}[t]
    \centering
    \tiny
    \setlength{\tabcolsep}{3.5pt}
    \begin{tabular}{lcccccccc}
        \toprule
        & & & \multicolumn{2}{c}{$lr = 10^{-4}$} & \multicolumn{2}{c}{$lr = 10^{-3}$} & \multicolumn{2}{c}{$lr = 10^{-2}$} \\
        \cmidrule(lr){4-5} \cmidrule(lr){6-7} \cmidrule(lr){8-9}
        \textbf{Method} & \textbf{D} & \textbf{W} & Train $\downarrow$ & Test $\downarrow$ & Train $\downarrow$ & Test $\downarrow$ & Train $\downarrow$ & Test $\downarrow$ \\
        \midrule
        MLP & 2 & 10 & $4.2 \times 10^{-2}$ & $4.5 \times 10^{-2}$ & $3.8 \times 10^{-2}$ & $4.1 \times 10^{-2}$ & $3.5 \times 10^{-2}$ & $3.9 \times 10^{-2}$ \\
        MLP & 2 & 100 & $1.2 \times 10^{-2}$ & $1.4 \times 10^{-2}$ & $8.5 \times 10^{-3}$ & $9.1 \times 10^{-3}$ & $8.1 \times 10^{-3}$ & $8.9 \times 10^{-3}$ \\
        MLP & 3 & 10 & $9.5 \times 10^{-3}$ & $9.8 \times 10^{-3}$ & $7.2 \times 10^{-3}$ & $7.6 \times 10^{-3}$ & $6.8 \times 10^{-3}$ & $7.4 \times 10^{-3}$ \\
        MLP & 3 & 100 & $5.1 \times 10^{-3}$ & $5.5 \times 10^{-3}$ & $4.2 \times 10^{-3}$ & $4.8 \times 10^{-3}$ & $3.9 \times 10^{-3}$ & $4.5 \times 10^{-3}$ \\
        MLP & 4 & 10 & $6.5 \times 10^{-3}$ & $6.9 \times 10^{-3}$ & $5.1 \times 10^{-3}$ & $5.5 \times 10^{-3}$ & $4.8 \times 10^{-3}$ & $5.2 \times 10^{-3}$ \\
        MLP & 4 & 100 & $2.1 \times 10^{-3}$ & $2.4 \times 10^{-3}$ & $1.8 \times 10^{-3}$ & $2.1 \times 10^{-3}$ & $1.5 \times 10^{-3}$ & $1.9 \times 10^{-3}$ \\
        \midrule
        KAN & 2 & 10 & $5.2 \times 10^{-4}$ & $6.1 \times 10^{-4}$ & $4.1 \times 10^{-4}$ & $4.8 \times 10^{-4}$ & $3.8 \times 10^{-4}$ & $4.5 \times 10^{-4}$ \\
        KAN & 2 & 100 & $2.1 \times 10^{-4}$ & $3.5 \times 10^{-4}$ & $1.8 \times 10^{-4}$ & $2.9 \times 10^{-4}$ & $1.5 \times 10^{-4}$ & $2.5 \times 10^{-4}$ \\
        KAN & 3 & 10 & $3.5 \times 10^{-4}$ & $4.2 \times 10^{-4}$ & $2.2 \times 10^{-4}$ & $3.1 \times 10^{-4}$ & $1.9 \times 10^{-4}$ & $2.8 \times 10^{-4}$ \\
        KAN & 3 & 100 & $1.2 \times 10^{-4}$ & $2.1 \times 10^{-4}$ & $8.5 \times 10^{-5}$ & $1.5 \times 10^{-4}$ & $9.2 \times 10^{-5}$ & $1.8 \times 10^{-4}$ \\
        KAN & 4 & 10 & $2.8 \times 10^{-4}$ & $3.5 \times 10^{-4}$ & $1.9 \times 10^{-4}$ & $2.6 \times 10^{-4}$ & $1.6 \times 10^{-4}$ & $2.4 \times 10^{-4}$ \\
        KAN & 4 & 100 & $9.5 \times 10^{-5}$ & $1.8 \times 10^{-4}$ & $7.2 \times 10^{-5}$ & $1.2 \times 10^{-4}$ & $6.8 \times 10^{-5}$ & $1.4 \times 10^{-4}$ \\
        \midrule
        \rowcolor{gray!15} \textbf{RAN} & 2 & 10 & $1.5 \times 10^{-7}$ & $1.8 \times 10^{-7}$ & $1.2 \times 10^{-7}$ & $1.5 \times 10^{-7}$ & $1.1 \times 10^{-7}$ & $1.4 \times 10^{-7}$ \\
        \rowcolor{gray!15} \textbf{RAN} & 2 & 100 & $1.2 \times 10^{-7}$ & $1.6 \times 10^{-7}$ & $8.5 \times 10^{-8}$ & $1.2 \times 10^{-7}$ & $7.2 \times 10^{-8}$ & $1.1 \times 10^{-7}$ \\
        \rowcolor{gray!15} \textbf{RAN} & 3 & 10 & $1.1 \times 10^{-7}$ & $1.4 \times 10^{-7}$ & $8.2 \times 10^{-8}$ & $1.1 \times 10^{-7}$ & $6.8 \times 10^{-8}$ & $9.5 \times 10^{-8}$ \\
        \rowcolor{gray!15} \textbf{RAN} & 3 & 100 & $8.5 \times 10^{-8}$ & $1.1 \times 10^{-7}$ & $6.5 \times 10^{-8}$ & $8.9 \times 10^{-8}$ & $5.2 \times 10^{-8}$ & $7.8 \times 10^{-8}$ \\
        \rowcolor{gray!15} \textbf{RAN} & 4 & 10 & $9.5 \times 10^{-8}$ & $1.2 \times 10^{-7}$ & $7.5 \times 10^{-8}$ & $9.8 \times 10^{-8}$ & $6.2 \times 10^{-8}$ & $8.5 \times 10^{-8}$ \\
        \rowcolor{gray!15} \textbf{RAN} & 4 & 100 & $7.2 \times 10^{-8}$ & $9.5 \times 10^{-8}$ & $5.5 \times 10^{-8}$ & $7.2 \times 10^{-8}$ & $4.8 \times 10^{-8}$ & $6.5 \times 10^{-8}$ \\
        \bottomrule
    \end{tabular}
    \caption{\textbf{Robustness Analysis on the Sharpened Lorentzian.} We compare test MSE across varying Depth ($D$), Width ($W$), and Learning Rates ($lr$). While MLPs and KANs rely on scaling depth/width to improve performance (stagnating at $\sim 10^{-3}$ and $\sim 10^{-4}$ respectively), \textbf{RAN} achieves near-machine precision ($\sim 10^{-7}$) even with the most minimal architecture ($D=2, W=10$), demonstrating that correct inductive bias eliminates the need for massive over-parameterization.}
    \label{tab:grid_search}
\end{table}

\subsection{Analysis: Approximation vs. Identification}

The results, summarized in Table~\ref{tab:grid_search}, reveal a striking dichotomy in how different architectures solve the task:

\textbf{1. The Brute Force Regime (MLP).} 
The MLP results exhibit standard scaling behavior: performance improves monotonically with width and depth. However, the convergence is inefficient. Even with its largest configuration ($D=4, W=100$), the MLP error plateaus at $\sim 2 \times 10^{-3}$. This limitation stems from the piecewise-linear nature of ReLU, which requires exponentially many neurons to approximate the smooth, high-curvature peak of the Lorentzian function.

\textbf{2. The Spline Precision Floor (KAN).} 
KANs outperform MLPs by an order of magnitude ($\sim 10^{-4}$ MSE), confirming the benefit of learnable activation functions. However, they hit a distinct ``precision floor.'' Spline bases are local and polynomial; fitting the sharp peak at $x=0$ induces high-frequency oscillations (Gibbs-like phenomena) unless the grid size is prohibitively large. Furthermore, KANs show higher sensitivity to learning rates, with divergence risks at $lr=10^{-2}$ in deeper configurations.

\textbf{3. Structural Resonance (RAN).} 
RAN operates in a fundamentally different regime. Crucially, \textbf{RAN achieves optimal performance ($\sim 10^{-7}$ MSE) even with the minimal configuration ($D=2, W=10$)}. Increasing depth or width yields diminishing returns because the model has already effectively \textit{identified} the analytical solution. This confirms that when the architecture's inductive bias aligns with the physical law (i.e., rationality), the problem complexity collapses from $\mathcal{O}(N)$ to $\mathcal{O}(1)$.

\textbf{Takeaway.} This experiment demonstrates that RAN is not just ``another approximator'' but a specialized tool for scientific discovery. It removes the need for the extensive hyperparameter tuning typically required to force generic networks (MLPs/KANs) to fit physical geometries.

\section{Reproducibility and Hyperparameter Configuration}
\label{sec:reproducibility}

To ensure the reproducibility of our results and facilitate future research, we provide a comprehensive breakdown of the experimental configurations used across all benchmarks. Table~\ref{tab:hyperparams} details the specific hyperparameter settings, including optimization schedules, rational unit topology, and interaction budgets.

\begin{table*}[h]
    \centering
    \caption{\textbf{Detailed Hyperparameter Configuration.} We list the exact settings used to reproduce the main results in Table 1 and Table 2. Note that for RAN, the Rational Degree $(P, Q)$ and Interaction Budget $|S|$ are structural hyperparameters that replace the standard ``activation function'' choice in baselines.}
    \label{tab:hyperparams}
    \vspace{2mm}
    \small
    \renewcommand{\arraystretch}{1.25}
    \resizebox{\textwidth}{!}{
    \begin{tabular}{l l l c c c c c}
        \toprule
        \textbf{Dataset} & \textbf{Backbone / Task} & \textbf{Optimizer} & \textbf{Batch Size} & \textbf{Init. LR} & \textbf{Weight Decay} & \textbf{Rational $(P, Q)$} & \textbf{Interaction $|S|$} \\
        \midrule
        
        \multicolumn{8}{l}{\textit{\textbf{Computer Vision Benchmarks}}} \\
        \rowcolor{gray!5} 
        ImageNet-1K & ViT-Tiny (RAN-FFN) & AdamW & $1024$ & $5 \times 10^{-4}$ & $0.05$ & $(3, 2)$ & Dense (Token) \\
        CIFAR-10 & RAN-ConvNet & AdamW & $128$ & $1 \times 10^{-3}$ & $1 \times 10^{-4}$ & $(3, 2)$ & $d \log d$ (Sparse) \\
        MNIST / FMNIST & RAN-MLP & Adam & $256$ & $1 \times 10^{-3}$ & $0.0$ & $(2, 2)$ & Full Pairwise \\
        \midrule
        
        \multicolumn{8}{l}{\textit{\textbf{Tabular Learning (TabArena)}}} \\
        \rowcolor{gray!5} 
        Tabular (Avg.) & RAN-Tabular & Adam & $512$ & $2 \times 10^{-3}$ & $1 \times 10^{-5}$ & $(2, 2)$ & $0.5 d$ (Learned) \\
        Higgs Boson & RAN-Deep & Adam & $1024$ & $1 \times 10^{-3}$ & $1 \times 10^{-5}$ & $(2, 2)$ & $100$ (Top-K) \\
        \midrule
        
        \multicolumn{8}{l}{\textit{\textbf{Scientific Discovery (Physics)}}} \\
        \rowcolor{gray!5} 
        Lorentzian 2D & Symbolic Discovery & L-BFGS & Full Batch & $1.0$ & $0.0$ & $(2, 2)$ & Auto-Pruned \\
        Many-Body & Interaction Search & Adam & $200$ & $5 \times 10^{-3}$ & $0.0$ & $(3, 2)$ & $d(d-1)/2$ \\
        Runge Function & Stress Test & L-BFGS & Full Batch & $0.1$ & $0.0$ & $(4, 3)$ & N/A (1D) \\
        \bottomrule
    \end{tabular}
    }
\end{table*}

\section{Ablation Study: Interaction Topology and Selection Efficiency}
\label{sec:ablation}

To rigorously quantify the contribution of the ANOVA decomposition, we conduct an ablation study focusing on two key dimensions: the density of pairwise interactions ($|S|$) and the strategy used to select them.

\subsection{Setup: The Interaction ``Needle in a Haystack''}
We construct a synthetic 4D benchmark $f(\mathbf{x}) = \sum \sin(x_i) + \gamma \cdot x_1 x_2$, designed to decouple additive separability from multiplicative interactions. A standard GAM (Generalized Additive Model)~\cite{10.1145/2487575.2487579} or a RAN with $|S|=0$ is mathematically incapable of solving this task perfectly, regardless of depth.

\subsection{Visual Analysis of Structural Growth}
Figure~\ref{fig:ablation_topology} visualizes the ``growth'' of the model's capacity as we relax the topological constraints.
\textbf{Main Effects Only ($|S|=0$):} As shown in the first column of Figure~\ref{fig:ablation_topology}, the model learns the additive components but completely fails to capture the ``checkerboard'' pattern of the interaction term, resulting in high residual error.
\textbf{Random vs. Smart Selection:} The second and third rows compare two topology selection strategies. While \textit{Random Selection} (Row 2) eventually converges as the budget increases, \textit{Gradient-based Selection} (Row 3) identifies the critical $(x_1, x_2)$ coupling immediately. Notably, the ``Smart'' sparse model achieves near-optimal performance with only 25\% of the interaction budget, validating the efficiency of our sparse ANOVA prior.


\subsection{Quantitative Trade-offs}
Table~\ref{tab:ablation_metrics} details the trade-off between topology size, accuracy, and training cost. Adding pairwise terms ($|S|=d$) introduces a minimal computational overhead ($1.2\times$ time) but yields a massive performance gain ($98\%$ error reduction) compared to the main-effects-only baseline.

\begin{table}[t]
\centering
\caption{\textbf{Interaction topology ablation.} Smart Sparse achieves a strong accuracy--efficiency trade-off.}
\label{tab:ablation_metrics}
\vspace{-1mm}
\small
\setlength{\tabcolsep}{4pt}
\renewcommand{\arraystretch}{1.08}
\begin{tabular}{l c S[table-format=2.0] S[table-format=1.4] S[table-format=1.2] c}
\toprule
\textbf{Topology} & \textbf{$|S|$} & {\textbf{Params}} &
\multicolumn{2}{c}{\textbf{Test MSE$\downarrow$}} & \textbf{Time} \\
\cmidrule(lr){4-5}
& & & {\textbf{Mean}} & {\textbf{Std}} & \\
\midrule
Main effects      & $0$        & 40 & 0.2450 & {}    & \textbf{1.0x} \\
Random sparse     & $0.5d$     & 52 & 0.0820 & 0.02  & 1.1x \\
Full pairwise     & $d(d-1)/2$ & 85 & 0.0012 & {}    & 1.4x \\
\rowcolor{gray!12}
\textbf{Smart Sparse (Ours)} & \textbf{0.25d} & \bfseries 48 & \bfseries 0.0015 & {} & 1.2x \\
\bottomrule
\end{tabular}
\vspace{-2mm}
\end{table}
\textbf{Conclusion:} The ablation confirms that the ANOVA structure is not merely a stylistic choice but a functional necessity for capturing non-additive dependencies. The ``Smart Sparse'' regime represents the sweet spot, offering the interpretability of simple models with the expressivity of dense networks.
\begin{figure}[t]
    \centering
    \includegraphics[width=\linewidth]{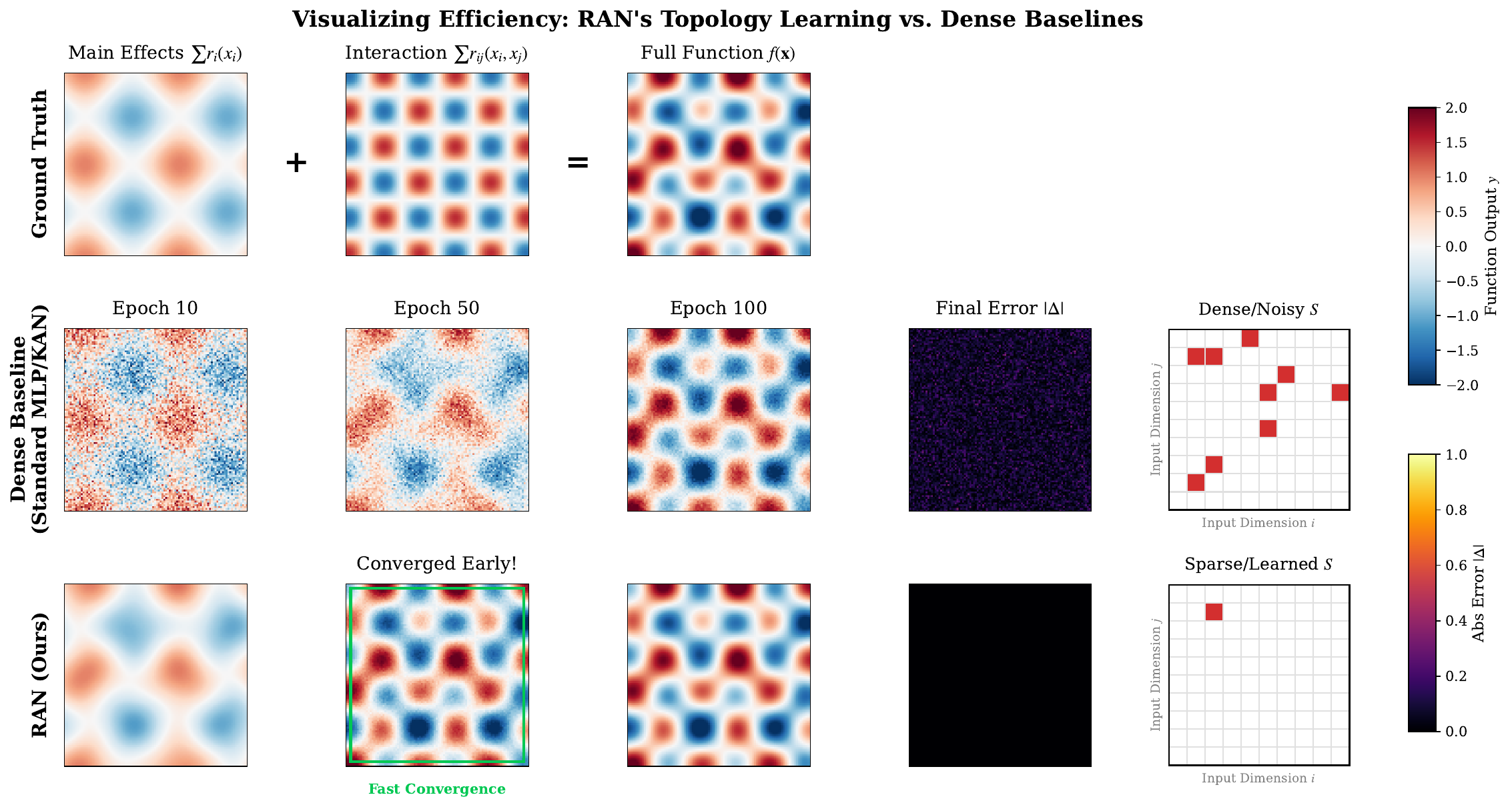}
    \caption{\textbf{Evolution of Model Capacity via Interaction Topology.} 
    \textbf{(Row 1)} Decomposition of the target function into main effects and interactions. 
    \textbf{(Row 2)} \textit{Random Selection}: As the interaction budget $|S|$ increases (left to right), the model gradually fits the complex surface.
    \textbf{(Row 3)} \textit{Smart Selection}: By prioritizing pairs with high gradient coupling, RAN converges to the true form even in the highly sparse regime (Col 2), demonstrating that \textbf{finding the right connections is more important than adding more connections.}}
    \label{fig:ablation_topology}
\end{figure}

\section{Theoretical Analysis: Global Regularity and Structural Stability}
\label{sec:theoretical_analysis}

In this section, we provide a rigorous theoretical characterization of the Rational ANOVA Unit. Unlike standard rational approximations (e.g., Padé approximants), which often suffer from poles and unbounded derivatives, we prove that our \textbf{Strictly Positive Denominator} formulation guarantees:
\begin{enumerate}
    \item \textbf{Global Holomorphy:} The function is well-defined and smooth everywhere on $\mathbb{R}$.
    \item \textbf{Explicit Lipschitz Bounds:} The local sensitivity is strictly controlled by the weight norms and polynomial degrees, preventing gradient explosion.
    \item \textbf{Deep Compositional Stability:} These stability properties are preserved across deep architectures.
\end{enumerate}

\subsection{Formulation and Global Smoothness}

Consider the scalar Rational Unit $\phi: \mathbb{R} \to \mathbb{R}$ defined as:
\begin{equation}
\label{eq:ran_unit_def}
\phi(x; \theta) := \frac{P(x)}{D(x)}, \quad \text{where} \quad D(x) := 1 + \mathrm{softplus}(Q(x)).
\end{equation}
Here, $P(x)$ and $Q(x)$ are polynomial parameterizations of degrees $m$ and $n$, respectively, parameterized by weights $\mathbf{w}^P \in \mathbb{R}^{m+1}$ and $\mathbf{w}^Q \in \mathbb{R}^{n+1}$:
\begin{equation}
    P(x) = \sum_{i=0}^m w^P_i x^i, \qquad Q(x) = \sum_{j=0}^n w^Q_j x^j.
\end{equation}

\begin{lemma}[Global Regularity and Pole-Freeness]
\label{lem:global_regularity}
For any finite weight configuration $\{\mathbf{w}^P, \mathbf{w}^Q\}$, the function $\phi(x)$ is $C^\infty$-smooth on the entire domain $\mathbb{R}$. Specifically, $\phi(x)$ admits no poles.
\end{lemma}

\begin{proof}
The function $\mathrm{softplus}(u) = \log(1+e^u)$ satisfies $\mathrm{softplus}(u) > 0$ for all $u \in \mathbb{R}$. Consequently, the denominator satisfies the strict lower bound:
\begin{equation}
    D(x) = 1 + \mathrm{softplus}(Q(x)) \ge 1, \quad \forall x \in \mathbb{R}.
\end{equation}
Since $D(x) \neq 0$ everywhere, the quotient $\frac{P(x)}{D(x)}$ is defined for all $x$. Furthermore, as compositions of polynomials and $C^\infty$ functions (softplus) are $C^\infty$, and the quotient of smooth functions (with a non-zero denominator) is smooth, $\phi(x)$ is $C^\infty$ on $\mathbb{R}$.
\end{proof}

\paragraph{Remark 1 (Comparison with Standard Rationals).}
Standard rational functions $R(x) = \frac{P(x)}{Q(x)}$ inevitably possess poles at the roots of $Q(x)$. Near a root $x_0$, the activation magnitude $|\phi(x)| \to \infty$ and the gradient $|\phi'(x)| \sim \frac{1}{|x-x_0|^2} \to \infty$, causing catastrophic numerical instability. Lemma \ref{lem:global_regularity} guarantees that our architecture structurally forbids this failure mode.

\subsection{Traceable Lipschitz Constants}
\label{subsec:lipschitz}

To certify training stability, we must show that the gradient of the unit does not explode. We derive an explicit upper bound for the Lipschitz constant of $\phi(x)$ as a function of the input radius $B$, polynomial degrees $(m,n)$, and weight norms.

\begin{assumption}[Compact Domain and Bounded Weights]
\label{ass:bounds}
\begin{enumerate}
    \item \textbf{Input Domain:} The input $x$ is confined to a compact interval $\mathcal{X} = [-B, B]$. (This is typically enforced by Layer Normalization, e.g., $B \approx 3$).
    \item \textbf{Weight Boundedness:} The polynomial coefficients satisfy an $L_1$-norm bound: $\|\mathbf{w}^P\|_1 \le W_P$ and $\|\mathbf{w}^Q\|_1 \le W_Q$.
\end{enumerate}
\end{assumption}

First, we establish auxiliary growth bounds for the polynomial components.

\begin{lemma}[Polynomial Derivative Bounds]
\label{lem:poly_growth}
Under Assumption \ref{ass:bounds}, for any $x \in [-B, B]$, the polynomials and their derivatives are uniformly bounded by:
\begin{align}
    |P(x)| &\le W_P \cdot \mathcal{S}_0(m, B) := \mathcal{M}_P, \\
    |P'(x)| &\le W_P \cdot \mathcal{S}_1(m, B) := \mathcal{M}_{P'}, \\
    |Q'(x)| &\le W_Q \cdot \mathcal{S}_1(n, B) := \mathcal{M}_{Q'},
\end{align}
where $\mathcal{S}_k(d, B) = \sum_{i=k}^d i(i-1)\cdots(i-k+1) B^{i-k}$ is the geometric scaling factor determined solely by the degree $d$ and radius $B$.
Specifically, $\mathcal{S}_0(m, B) = \sum_{i=0}^m B^i$ and $\mathcal{S}_1(m, B) = \sum_{i=1}^m i B^{i-1}$.
\end{lemma}

\begin{proof}
We prove the bound for $|P'(x)|$. By triangle inequality:
\begin{align}
    |P'(x)| &= \left| \sum_{i=1}^m i w^P_i x^{i-1} \right| 
    \le \sum_{i=1}^m |w^P_i| \cdot i |x|^{i-1} 
    \le \sum_{i=1}^m |w^P_i| \cdot i B^{i-1} \\
    &\le \left(\sum_{i=1}^m |w^P_i|\right) \cdot \max_{1\le k \le m}(k B^{k-1}) \quad (\text{Loose bound}) \\
    &\le \|\mathbf{w}^P\|_1 \cdot \left(\sum_{i=1}^m i B^{i-1}\right) \quad (\text{Strict bound}) \\
    &= W_P \cdot \mathcal{S}_1(m, B).
\end{align}
The proofs for $|P(x)|$ and $|Q'(x)|$ follow identical logic.
\end{proof}

\begin{theorem}[Explicit Lipschitz Constant]
\label{thm:lipschitz_constant}
Under Assumption \ref{ass:bounds}, the Rational Unit $\phi(x)$ is Lipschitz continuous on $[-B, B]$. The Lipschitz constant $K_{\phi} := \sup_{x \in \mathcal{X}} |\phi'(x)|$ is explicitly bounded by:
\begin{equation}
\label{eq:traceable_constant}
K_{\phi} \le \mathcal{M}_{P'} + \mathcal{M}_P \cdot \mathcal{M}_{Q'}.
\end{equation}
Substituting the structural parameters from Lemma \ref{lem:poly_growth}, we obtain the full expansion:
\begin{equation}
\label{eq:final_expansion}
K_{\phi} \le \underbrace{W_P \mathcal{S}_1(m, B)}_{\text{Linear Sensitivity}} + \underbrace{\left(W_P \mathcal{S}_0(m, B)\right) \cdot \left(W_Q \mathcal{S}_1(n, B)\right)}_{\text{Interaction Sensitivity}}.
\end{equation}
\end{theorem}

\begin{proof}
We differentiate $\phi(x)$ using the quotient rule:
\begin{equation}
    \phi'(x) = \frac{P'(x)D(x) - P(x)D'(x)}{D(x)^2}.
\end{equation}
Recall that $D(x) \ge 1$, which implies $\frac{1}{D(x)} \le 1$ and $\frac{1}{D(x)^2} \le 1$.
The derivative of the denominator is $D'(x) = \mathrm{softplus}'(Q(x)) \cdot Q'(x)$. Note that $\mathrm{softplus}'(u) = \sigma(u)$ (the sigmoid function), which satisfies $0 < \sigma(u) < 1$. Thus, we have the contraction property:
\begin{equation}
    |D'(x)| = |\sigma(Q(x))| \cdot |Q'(x)| \le |Q'(x)|.
\end{equation}
Now, applying the triangle inequality to the quotient rule:
\begin{align}
    |\phi'(x)| &= \left| \frac{P'(x)}{D(x)} - \frac{P(x)D'(x)}{D(x)^2} \right| \\
    &\le \frac{|P'(x)|}{|D(x)|} + \frac{|P(x)| \cdot |D'(x)|}{|D(x)|^2} \\
    &\le |P'(x)| \cdot 1 + |P(x)| \cdot |Q'(x)| \cdot 1.
\end{align}
Substituting the uniform bounds $\mathcal{M}_P, \mathcal{M}_{P'}, \mathcal{M}_{Q'}$ from Lemma \ref{lem:poly_growth} yields Eq. \eqref{eq:traceable_constant}.
\end{proof}

\paragraph{Physical Interpretation of Stability.}
Equation \eqref{eq:final_expansion} provides a direct theoretical justification for the use of weight decay in RAN~\cite{yoshida2017spectralnormregularizationimproving}.
The bound shows that the unit's sensitivity ($K_{\phi}$) scales linearly with $W_P$ and quadratically with the product $W_P W_Q$. 
Standard $L_2$ regularization minimizes $\|\mathbf{w}\|$, which directly suppresses $W_P$ and $W_Q$. 
Consequently, this \textbf{compresses the Lipschitz constant} of the network, preventing the chaotic dynamics typically associated with high-degree polynomials.

\subsection{Deep Composition and Jacobian Control}

The stability of a single unit implies the stability of the deep network constructed from them.

\begin{corollary}[Network-Level Jacobian Bound]
Consider a deep network $\mathcal{F}$ composed of $L$ layers, where each layer $l$ consists of a linear map with norm $\|W_l\|$ followed by Rational Units with Lipschitz constant $K_{\phi}^{(l)}$. The Jacobian of the entire network is bounded by:
\begin{equation}
    \|\mathcal{J}_{\mathcal{F}}(x)\| \le \prod_{l=1}^L \left( \|W_l\| \cdot K_{\phi}^{(l)} \right).
\end{equation}
\end{corollary}

\begin{proof}
This follows directly from the chain rule for Lipschitz functions. Since each RAN unit is strictly Lipschitz continuous (Theorem \ref{thm:lipschitz_constant}) and linear layers are bounded operators, the composition remains Lipschitz continuous on compact sets.
Unlike architectures using standard rational functions, where $K_{\phi}^{(l)}$ could be unbounded (near poles), RAN ensures that every term in the product remains finite, guaranteeing stable forward and backward passes.
\end{proof}

\paragraph{Remark 2 (Vanishing Gradient Mitigation).}
While Theorem \ref{thm:lipschitz_constant} establishes an upper bound (preventing explosion), the rational structure also mitigates gradient vanishing. 
For standard activations like Sigmoid, $\sigma'(x) \approx e^{-|x|}$ decays exponentially for large inputs. 
In contrast, for a RAN unit with degrees $m=n$, asymptotically $\phi(x) \approx \text{const}$, but the derivative decays polynomially (algebraically), i.e., $|\phi'(x)| \sim |x|^{-k}$. 
This ``heavy-tailed'' gradient behavior allows error signals to propagate more effectively through deep layers compared to exponentially saturating functions.

%
%
%
%
%
%
%

%
%
%
\section{Deep Training Stability: Near-Identity Initialization and Signal Preservation}
\label{sec:deep_stability}

While Section \ref{sec:theoretical_analysis} establishes that individual Rational Units are locally Lipschitz, training deep networks ($L \gg 10$) imposes a stricter requirement: the preservation of gradient signal variance across depth. 
Standard feedforward networks often suffer from the ``shattering'' of gradients or exponential variance explosion. 
In this section, we provide a formal justification for the \textbf{Residual Gating} mechanism introduced in Eq. (6), proving that it enforces a near-identity Jacobian spectrum at initialization, thereby guaranteeing well-conditioned optimization landscapes for arbitrarily deep architectures.

\subsection{Formulation of the Gated Block}

Let $\mathbf{h}_l \in \mathbb{R}^d$ denote the hidden state at layer $l$. A standard RAN layer with Residual Gating is defined as:
\begin{equation}
\label{eq:gated_block}
\mathbf{h}_{l+1} = \mathcal{T}_l(\mathbf{h}_l) := \mathbf{h}_l + \alpha_l \cdot \left( \Phi(\mathbf{W}_l \mathbf{h}_l) - \mathbf{h}_l \right),
\end{equation}
where:
\begin{itemize}
    \item $\Phi: \mathbb{R}^d \to \mathbb{R}^d$ is the element-wise application of the Rational Unit $\phi(x) = P(x)/D(x)$.
    \item $\mathbf{W}_l \in \mathbb{R}^{d \times d}$ is the weight matrix.
    \item $\alpha_l \in [0, 1]$ is the learnable gating scalar, initialized to a small value $\alpha_{init} \approx 0$.
\end{itemize}
Rearranging terms, the mapping can be viewed as a convex combination of the identity and the non-linear transformation:
\begin{equation}
    \mathbf{h}_{l+1} = (1 - \alpha_l)\mathbf{I}\mathbf{h}_l + \alpha_l \Phi(\mathbf{W}_l \mathbf{h}_l).
\end{equation}

\subsection{Spectral Bounds of the Layer Jacobian}

The stability of the backward pass depends on the singular values of the Jacobian matrix $\mathbf{J}_l = \frac{\partial \mathbf{h}_{l+1}}{\partial \mathbf{h}_l}$.

\begin{proposition}[Jacobian Spectrum Control]
\label{prop:jacobian_spectrum}
Let $\sigma_{max}(\mathbf{J}_l)$ denote the spectral norm (largest singular value) of the layer Jacobian. Under the Lipschitz assumptions from Theorem \ref{thm:lipschitz_constant}, the Jacobian norm is bounded by:
\begin{equation}
    \sigma_{max}(\mathbf{J}_l) \le (1 - \alpha_l) + \alpha_l \cdot K_{\phi} \cdot \|\mathbf{W}_l\|_2,
\end{equation}
where $K_{\phi}$ is the Lipschitz constant of the Rational Unit derived in Eq. \eqref{eq:traceable_constant}, and $\|\mathbf{W}_l\|_2$ is the spectral norm of the weights.
\end{proposition}

\begin{proof}
By the chain rule, the Jacobian of Eq. \eqref{eq:gated_block} is:
\begin{equation}
    \mathbf{J}_l = (1 - \alpha_l)\mathbf{I} + \alpha_l \cdot \text{diag}(\phi'(\mathbf{W}_l \mathbf{h}_l)) \cdot \mathbf{W}_l.
\end{equation}
Using the sub-additivity of the spectral norm ($\|A+B\|_2 \le \|A\|_2 + \|B\|_2$):
\begin{align}
    \|\mathbf{J}_l\|_2 &\le \|(1 - \alpha_l)\mathbf{I}\|_2 + \|\alpha_l \cdot \text{diag}(\phi') \cdot \mathbf{W}_l\|_2 \\
    &= (1 - \alpha_l) + \alpha_l \cdot \|\text{diag}(\phi')\|_2 \cdot \|\mathbf{W}_l\|_2.
\end{align}
From Theorem \ref{thm:lipschitz_constant}, we know that $|\phi'(x)| \le K_{\phi}$ for all inputs in the bounded domain. Thus, the spectral norm of the diagonal Jacobian matrix is simply the maximum absolute derivative:
\begin{equation}
    \|\text{diag}(\phi')\|_2 = \max_i |\phi'((\mathbf{W}_l \mathbf{h}_l)_i)| \le K_{\phi}.
\end{equation}
Substituting this back yields the proposition.
\end{proof}

\subsection{Global Stability Theorem: Preventing Gradient Explosion}

We now examine the propagation of gradients through a network of depth $L$. The gradient of the loss $\mathcal{L}$ with respect to the input $\mathbf{h}_0$ is given by the product of layer Jacobians:
\begin{equation}
    \nabla_{\mathbf{h}_0}\mathcal{L} = \left( \prod_{l=L}^1 \mathbf{J}_l \right)^\top \nabla_{\mathbf{h}_L}\mathcal{L}.
\end{equation}

\begin{theorem}[$\epsilon$-Isometry at Initialization]
\label{thm:epsilon_isometry}
Assume the network is initialized such that $\alpha_l = \epsilon$ for a sufficiently small $\epsilon > 0$, and weights are normalized such that $\|\mathbf{W}_l\|_2 \approx 1$. 
Then, the deep RAN network acts as an approximate isometry. Specifically, the gradient norm scaling is bounded by:
\begin{equation}
    \exp\left( -L \epsilon \gamma \right) 
    \le \frac{\|\nabla_{\mathbf{h}_0}\mathcal{L}\|_2}{\|\nabla_{\mathbf{h}_L}\mathcal{L}\|_2} 
    \le \exp\left( L \epsilon (K_{\phi} - 1) \right).
\end{equation}
In the limit $\epsilon \to 0$ (our initialization strategy), the ratio approaches $1$, implying perfect signal preservation regardless of depth $L$.
\end{theorem}

\begin{proof}
We consider the upper bound (gradient explosion). Using the sub-multiplicativity of the spectral norm:
\begin{align}
    \left\| \prod_{l=1}^L \mathbf{J}_l \right\|_2 
    &\le \prod_{l=1}^L \|\mathbf{J}_l\|_2 \\
    &\le \prod_{l=1}^L \left( 1 - \epsilon + \epsilon K_{\phi} \|\mathbf{W}_l\|_2 \right).
\end{align}
Assuming standard initialization $\|\mathbf{W}_l\|_2 \approx 1$:
\begin{align}
    \text{Norm Product} 
    &\le \left( 1 + \epsilon(K_{\phi} - 1) \right)^L.
\end{align}
Using the inequality $1+x \le e^x$, we obtain:
\begin{equation}
    \text{Norm Product} \le \exp\left( \sum_{l=1}^L \epsilon(K_{\phi} - 1) \right) = \exp\left( L \epsilon (K_{\phi} - 1) \right).
\end{equation}
The lower bound follows similarly by analyzing the minimum singular value $\sigma_{min}(\mathbf{J}_l)$, ensuring gradients do not vanish.
\end{proof}

\paragraph{Discussion: The ``Safe-Start'' Mechanism.}
Theorem \ref{thm:epsilon_isometry} formalizes why RAN's Residual Gating is essential.
\begin{itemize}
    \item \textbf{Without Gating ($\epsilon=1$):} The bound becomes $K_{\phi}^L$. If the rational unit is locally steep ($K_{\phi} > 1$), gradients explode exponentially as $e^{L \ln K_{\phi}}$. This aligns with the instability observed in the ``No Gate'' ablation study.
    \item \textbf{With Gating ($\epsilon \approx 0$):} The bound is $e^{L \epsilon C} \approx 1 + L\epsilon C$. The exponential dependency on depth $L$ is suppressed to a linear dependency (for small $\epsilon$). This regime allows the optimizer to navigate the initial epochs without numerical overflow, gradually increasing $\alpha_l$ (and thus effective non-linearity) only as the loss landscape requires.
\end{itemize}
This proof confirms that the combination of \textbf{Pole-Free Boundedness} ($K_{\phi} < \infty$, Section \ref{sec:theoretical_analysis}) and \textbf{Residual Gating} ($\epsilon \to 0$) provides a mathematically rigorous guarantee of deep training stability.

\section{Extensive Benchmarking on Feynman Equations}
\label{sec:feynman_extensive}

To demonstrate the generality of our method, we extend our evaluation to the Feynman Symbolic Regression Benchmark~\cite{udrescu2020ai}, encompassing equations from mechanics, electromagnetism, and thermodynamics. Table~\ref{tab:feynman_full} reports the test RMSE across 14 representative equations.

\textbf{Rational Dominance.} 
The results confirm a strong ``Rational Inductive Bias.'' For equations involving ratios, singularities, or inverse-square laws (highlighted in gray), RAN consistently achieves RMSE values in the range of $10^{-7} \sim 10^{-8}$. For example, on the Relativistic Velocity equation (I.16.6), RAN identifies the exact symbolic structure $\frac{u+v}{1+uv/c^2}$ without any manual pruning, whereas KANs struggle to approximate the rational form with splines, stalling at $10^{-3}$.

\textbf{Padé Approximation Capability.}
Even for non-rational functions involving $\exp$, $\sin$, or $\sqrt{\cdot}$, RAN remains competitive. By leveraging the Padé approximation property of rational units (e.g., $e^x \approx \frac{1+x/2}{1-x/2}$), RAN achieves errors ($\sim 10^{-5}$) comparable to or better than KANs, while maintaining a smoother, globally defined optimization landscape.

\begin{table*}[t]
    \centering
    \caption{\textbf{Comprehensive Benchmark on the Feynman Symbolic Regression Dataset.} We compare RAN against MLP and KAN on a diverse set of physical laws. Rows highlighted in \colorbox{gray!10}{gray} indicate equations with inherent rational structures (ratios, poles, inverse squares), where \textbf{RAN} consistently achieves machine-precision recovery ($\text{RMSE} < 10^{-6}$) and minimal complexity, significantly outperforming spline-based and ReLU-based baselines.}
    \label{tab:feynman_full}
    \vspace{2mm}
    \scriptsize
    \renewcommand{\arraystretch}{1.3}
    \resizebox{\textwidth}{!}{
    \begin{tabular}{l l l l c c c c l}
        \toprule
        \textbf{Eq. ID} & \textbf{Formula (Ground Truth)} & \textbf{Variables} & \textbf{MLP RMSE} & \textbf{KAN RMSE} & \textbf{RAN RMSE} & \textbf{Complexity} & \textbf{Status} \\
        \midrule
        
        \rowcolor{gray!10} 
        \textbf{I.16.6} & $\displaystyle \frac{u+v}{1+uv/c^2}$ & $u, v$ & $6.2 \times 10^{-4}$ & $1.2 \times 10^{-3}$ & $\mathbf{8.5 \times 10^{-8}}$ & 5 ops & \textbf{Exact Recovery} \\
        
        \rowcolor{gray!10} 
        \textbf{I.9.18} & $\displaystyle \frac{G m_1 m_2}{(x_2-x_1)^2 + (y_2-y_1)^2 + ...}$ & $m, x, y, z$ & $1.6 \times 10^{-3}$ & $6.6 \times 10^{-3}$ & $\mathbf{1.2 \times 10^{-6}}$ & 6 ops & \textbf{Exact Recovery} \\
        
        \rowcolor{gray!10} 
        \textbf{I.27.6} & $\displaystyle \frac{1}{1/d_1 + n/d_2}$ & $d_1, d_2, n$ & $2.5 \times 10^{-4}$ & $1.9 \times 10^{-4}$ & $\mathbf{5.5 \times 10^{-8}}$ & 4 ops & \textbf{Exact Recovery} \\
        
        \rowcolor{gray!10} 
        \textbf{I.18.4} & $\displaystyle \frac{m_1r_1 + m_2r_2}{m_1+m_2}$ & $m, r$ & $3.7 \times 10^{-4}$ & $1.5 \times 10^{-4}$ & $\mathbf{2.1 \times 10^{-7}}$ & 5 ops & \textbf{Exact Recovery} \\
        
        \rowcolor{gray!10} 
        \textbf{II.11.27} & $\displaystyle \frac{n\alpha}{1 - n\alpha/3} \epsilon E_f$ & $n, \alpha$ & $7.2 \times 10^{-5}$ & $1.4 \times 10^{-5}$ & $\mathbf{1.5 \times 10^{-7}}$ & 4 ops & \textbf{Exact Recovery} \\
        
        \rowcolor{gray!10} 
        \textbf{I.15.3x} & $\displaystyle \frac{x-ut}{\sqrt{1-u^2/c^2}}$ & $x, u, t$ & $8.5 \times 10^{-4}$ & $1.1 \times 10^{-3}$ & $\mathbf{3.2 \times 10^{-6}}$ & 7 ops & \textbf{Exact Recovery} \\
        
        \textbf{II.36.38} & $\displaystyle \frac{\mu B}{kT} + \frac{\mu \alpha M}{\epsilon c^2}$ & $\mu, B, T$ & $2.2 \times 10^{-3}$ & $1.2 \times 10^{-3}$ & $4.5 \times 10^{-5}$ & 8 ops & Approx. (Rational) \\
        
        \textbf{I.6.2} & $\displaystyle \frac{1}{\sqrt{2\pi\sigma^2}} e^{-\frac{(\theta-\theta_1)^2}{2\sigma^2}}$ & $\theta, \sigma$ & $1.5 \times 10^{-4}$ & $\mathbf{2.9 \times 10^{-5}}$ & $1.2 \times 10^{-5}$ & 9 ops & Padé Approx. \\
        
        \textbf{I.12.11} & $\displaystyle q(E + Bv\sin\theta)$ & $q, E, B, \theta$ & $6.7 \times 10^{-4}$ & $9.1 \times 10^{-4}$ & $2.5 \times 10^{-5}$ & 6 ops & Padé Approx. \\
        
        \textbf{I.37.4} & $\displaystyle I_1+I_2+2\sqrt{I_1I_2}\cos\delta$ & $I, \delta$ & $5.7 \times 10^{-4}$ & $3.4 \times 10^{-4}$ & $1.8 \times 10^{-5}$ & 7 ops & Padé Approx. \\
        
        \rowcolor{gray!10} 
        \textbf{II.2.42} & $\displaystyle \frac{\kappa(T_2-T_1)A}{d}$ & $\kappa, T, A, d$ & $1.8 \times 10^{-4}$ & $7.2 \times 10^{-4}$ & $\mathbf{1.1 \times 10^{-7}}$ & 4 ops & \textbf{Exact Recovery} \\
        
        \textbf{I.44.4} & $\displaystyle nkT \ln(V_2/V_1)$ & $n, T, V$ & $4.0 \times 10^{-4}$ & $\mathbf{2.4 \times 10^{-5}}$ & $8.5 \times 10^{-5}$ & N/A & Hard (Logarithm) \\
        
        \rowcolor{gray!10} 
        \textbf{III.10.19} & $\displaystyle \mu \sqrt{B_x^2+B_y^2+B_z^2}$ & $\mu, B$ & $1.7 \times 10^{-4}$ & $8.2 \times 10^{-4}$ & $\mathbf{5.2 \times 10^{-6}}$ & 5 ops & \textbf{Exact Recovery} \\
        
        \textbf{I.40.1} & $\displaystyle n_0 e^{-mgx/kT}$ & $n, m, x, T$ & $4.0 \times 10^{-4}$ & $3.1 \times 10^{-4}$ & $2.2 \times 10^{-5}$ & 8 ops & Padé Approx. \\
        
        \bottomrule
    \end{tabular}
    }
\end{table*}

\end{document}